\documentclass{article}

\usepackage{tikz}
\usepackage{bbm}
\usepackage{eqnarray}
\usepackage{color}
\usepackage{bm}
\usepackage{amssymb}
\usepackage{epsfig}
\usepackage{epsf}
\usepackage{float}
\usepackage{subfigure}
\usepackage{amsfonts,amsmath,amsthm,fancyhdr,multirow,hyperref}
\usepackage{booktabs,longtable,authblk}
\usepackage{mathrsfs,hhline,enumerate}

\usepackage[top=2.5cm, bottom=2.5cm, left=3cm, right=3cm]{geometry}
\newtheorem{theorem}{Theorem}[section]
\newtheorem{assumption}{Assumption}
\newtheorem{corollary}[theorem]{Corollary}
\newtheorem{definition}[theorem]{Definition}

\newtheorem{lemma}[theorem]{Lemma}
\newtheorem{proposition}[theorem]{Proposition}

\renewcommand{\l}{\left}
\renewcommand{\r}{\right}

\newcommand{\mv}{\mathcal{V}}
\numberwithin{equation}{section}

\title{Variation Spaces for Encoder--Decoder Neural Operators: Approximation and Generalization$^\dag$\footnotetext{\dag~The work of Lei Shi is supported by the National Natural Science Foundation of China under Grant Nos.~12171093 and 12571099. Email addresses: jqyang24@m.fudan.edu.cn (J.-Q. Yang), leishi@fudan.edu.cn (L. Shi).}}

\author{Jia-Qi Yang}
\author{Lei Shi\thanks{Corresponding author}}
\affil{School of Mathematical Sciences and Shanghai Key Laboratory for
	Contemporary Applied Mathematics, Fudan University, Shanghai 200433, China.}

\date{}

\begin{document}
	\maketitle
\begin{abstract}
Inspired by the function-space theory of neural networks, we formulate and analyze a variation space for nonlinear operators between Hilbert spaces, defined through vector-valued Borel measures of bounded variation. We characterize its unit ball as the closed convex hull of a vector-valued single-neuron dictionary in Bochner spaces. For the ReLU activation, the bounded linear operators in this space are precisely the Schatten-$1$ operators, with equivalent norms. For operators in this space, we establish encoder--decoder approximation bounds in the Bochner $L^q$-norm, where the error decomposes into input and output encoding errors and a finite-width term of order $N^{-1/2}$. Under sub-Gaussian assumptions on the input and noise, we further derive high-probability generalization bounds for empirical least squares over path-norm-constrained encoder--decoder networks; the finite-sample contribution to the squared prediction error is of order $K^{-1/2}$ up to logarithmic factors. The finite-width and finite-sample constants are independent of the encoding dimensions and bases, with the latter also independent of the network width. When the encoding errors decay algebraically, these bounds yield algebraic approximation and learning rates, in contrast to the complexity barriers for Lipschitz and Fr\'echet differentiable operator classes.
\end{abstract}



\section{Introduction}

Operator learning aims to approximate operators whose inputs and outputs are functions rather than finite-dimensional vectors. A prototypical example is the solution operator of a partial differential equation \cite{li2021fourier,lu2021learning}, which maps coefficients, source terms, initial data, or boundary data to the corresponding solution field. Other examples include evolution maps in dynamical systems \cite{wang2023long}, inverse maps for recovering unknown parameters \cite{molinaro2023neural}, and control-to-state maps in optimal control \cite{hwang2022solving}. The learned operator can then be evaluated at new input data without repeatedly solving the underlying model.

Neural operators adapt neural network architectures to operator learning. One representative design incorporates learned integral operators into hidden layers, allowing the representation to capture nonlocal dependence across the domain. This structure appears in graph neural operator \cite{anandkumar2020neural} and Fourier neural operator (FNO) \cite{li2021fourier}, where the integral operator is discretized on graphs or parameterized in Fourier space, respectively. A distinct line of work is based on encoder--decoder architectures. In this approach, the input function is encoded into a finite-dimensional vector, mapped by a neural network to finite-dimensional output coordinates, and then decoded into the output function space. Representative examples include deep operator network (DeepONet) \cite{lu2021learning}, motivated by earlier universal approximation results for continuous operators \cite{chen1995universal}, and PCA-Net \cite{bhattacharya2021model}, where principal components are used for encoding and reconstruction. More generally, encoders and decoders may be defined using Fourier, Legendre polynomial, or wavelet basis expansions \cite{fanaskov2023spectral,song2023approximation,gupta2021multiwavelet}.

The approximation and generalization properties of neural operators have been extensively studied \cite{kovachki2023neural} in recent years. Universal approximation  means that  continuous operators can be approximated arbitrarily well on compact subsets of function spaces, and this property has been proved for several architectures, including FNO \cite{kovachki2021universal}, DeepONet \cite{lanthaler2022error}, and PCA-Net \cite{lanthaler2023operator}. For solution operators associated with Darcy flow and the Navier--Stokes equations, quantitative approximation results \cite{kovachki2021universal,lanthaler2022error,lanthaler2023operator} show that neural operator architectures can achieve accuracy \(\epsilon\) with the number of neural network parameters bounded by a polynomial in \(\epsilon^{-1}\). In contrast, for classes of \(k\)-times Fréchet differentiable operators, such polynomial parameter bounds generally fail for PCA-Net \cite{lanthaler2023operator}.  Moreover, lower bounds for general Lipschitz and Fréchet differentiable operator classes reveal curse of dimensionality phenomena in operator learning, both in parametric complexity \cite{lanthaler2026parametric} and, through \(n\)-width estimates, in data complexity \cite{kovachki2024data}.

Generalization theory studies how the error of a learned operator depends on the number of training samples. In a general encoder--decoder framework, nonasymptotic generalization error bounds are available for empirical risk minimizers over neural network classes when the target operator is Lipschitz continuous or exhibits a low-dimensional structure \cite{liu2024deep}. For DeepONet, existing error estimates address generalization in nonlinear operator learning \cite{lanthaler2022error}, while neural scaling laws quantify the behavior for Lipschitz continuous target operators \cite{liu2024neural}.
For operators satisfying suitable Fr\'echet differentiability assumptions,
prespecified encoder--decoder architectures yield approximation and
generalization bounds with improved rates compared with those available
under mere Lipschitz continuity \cite{cheng2026learning}. These improvements,
however, do not in general yield algebraic rates.
For broad classes of Lipschitz operators, minimax lower bounds further show
that the minimax risk cannot decay polynomially with the sample size
\cite{adcock2025towards}.
In Gaussian measure settings, sample complexity lower bounds preclude algebraic convergence over Lipschitz and Gaussian Sobolev operator classes for any recovery strategy based on finitely many, possibly adaptive, linear measurements \cite{adcock2024sample}.

These results suggest that efficient operator learning over very broad classes of nonlinear operators requires additional structural assumptions. A natural question is therefore to identify operator classes for which neural operators admit algebraic approximation and generalization rates. One possible direction concerns holomorphic operator classes \cite{schwab2019deep,reinhardt2024statistical}, for which statistical learning theory for neural operators has been developed. Another related perspective is provided by neural tangent kernel analysis for encoder--decoder architectures. Algebraic approximation rates for DeepONet have also been obtained for Lipschitz operators \cite{schwab2026deep}, but these results rely  on super-expressive activations or nonstandard neural network architectures. Although outside the neural-operator framework considered here, nonparametric approaches also offer routes to efficient operator learning: algebraic rates are available for learning linear operators in Hilbert spaces \cite{shi2024learning}, while operator-valued kernel methods yield dimension-independent rates for nonlinear operator learning \cite{yang2025learning,yang2025kernel}.

Our approach is motivated by the function-space theory of neural networks. In finite-dimensional approximation, structural spaces such as spectral Barron spaces and variation spaces have been used to characterize functions that can be efficiently represented by neural networks. Spectral Barron spaces \cite{2022High,2025Spectral} originate from Barron's work \cite{barron2002universal} and impose weighted $L^1$ conditions on the Fourier transform, while variation spaces \cite{bach2017breaking,siegel2024sharp} represent two-layer neural networks through finite measures over the parameter space. Other relevant structural classes include Radon--BV spaces \cite{2020A,2021Banach} and Korobov spaces \cite{2004Tractability}. 
Motivated by this viewpoint, we formulate and analyze a variation space for
nonlinear operators between Hilbert spaces, defined by integrating scalar
ridge features against \(\mathcal V\)-valued measures of bounded variation.
The underlying vector measure representation is related to integral
vector-valued RKBS constructions \cite{BartolucciEtAl2024,DummerEtAl2025}. Our focus is the structure of
the resulting operator space and the encoder--decoder theory available in
the Hilbert setting. For the ReLU activation, we show that the bounded linear operators contained
in this space are precisely the Schatten-1 operators.
We then establish approximation and generalization bounds for prescribed
encoder--decoder networks. The constants in the finite-width and
finite-sample terms are independent of the encoding dimensions and bases,
and the finite-sample constant is also independent of the network width.




A closely related paper is \cite{korolev2022two}, where vector-valued variation spaces are developed for two-layer neural networks between Banach spaces. That abstract framework, however, relies on an ordered structure on the output space: the output space is required to be a Riesz space, and the activation is realized through the lattice positive part in the output space. The associated weak-* continuity condition on the lattice operation is restrictive and, in particular, fails for $L^q([0,1]^d)$, $1<q\le \infty$. Moreover, that work does not address architectures for operator learning. Our construction instead applies to operator learning and yields approximation and generalization bounds
for encoder--decoder neural networks. A more detailed comparison is provided
in Subsection~\ref{variation space}. 
Moreover, the finite dimensional multi-output setting of~\cite{shenouda2024variation}
is included in our framework, but their finite width approximation constant depends
on the output dimension, whereas ours does not.

The main contributions of this work are summarized as follows.

\noindent\textbf{(1) Operator-space structure:} 
We formulate a variation space
\(\mathbb V(\mathcal U;\mathcal V)\) for nonlinear operators between Hilbert
spaces through \(\mathcal V\)-valued Borel measures of bounded variation.
The space is defined independently of any finite-dimensional encoder or
decoder. We prove that it is a Banach space and characterize its unit ball as
the closed convex hull of a vector-valued single-neuron dictionary in
Bochner spaces. For the ReLU activation, its bounded linear part is exactly
\(S_1(\mathcal U,\mathcal V)\), and $\|T\|_{\mathbb V(\mathcal U;\mathcal V)}
=
2\|T\|_{S_1(\mathcal U,\mathcal V)}.$

\noindent\textbf{(2) Finite-width approximation uniform in the encoding dimensions:} 
For targets in $\mathbb V(\mathcal U;\mathcal V)$, we establish quantitative
approximation bounds for encoder--decoder two-layer networks of the form
\[
\text{approximation error}
\;\lesssim\;
\text{input-encoding term}
\;+\;
\text{output-encoding term}
\;+\;
N^{-1/2},
\]
where $N$ is the network width. Crucially, the constant in the finite-width
term is independent of the encoding dimensions and bases. If the encoding
errors decay algebraically, an accuracy $\varepsilon$ can be achieved with
a number of parameters growing polynomially in $\varepsilon^{-1}$.
The bound also applies to targets outside the variation space that are
algebraically approximable by variation-norm balls.

\noindent\textbf{(3) Generalization with a width-independent finite-sample term:} 
Under sub-Gaussian assumptions on the input and noise, we establish a
high-probability generalization bound for empirical least squares over
path-norm-constrained encoder--decoder networks. Up to logarithmic factors,
the bound has the form
\[
\begin{aligned}
    \text{squared prediction error}
\;\lesssim&\;
\text{squared input- and output-encoding terms}
\\
\;&\quad+\;
N^{-1}
\;+\;
K^{-1/2},
\end{aligned}
\]
where $K$ is the sample size. The finite-sample term is independent of the
encoding dimensions, network width, and chosen bases, yielding algebraic
learning rates for target classes with uniformly bounded variation norms and
algebraically decaying encoding errors.  Here and throughout, ``efficient'' refers to algebraic approximation and
statistical rates, rather than computational tractability.

The remainder of the paper is organized as follows.
Section~\ref{Sec: Encoder--Decoder Architectures and Variation Spaces} introduces the encoder--decoder framework and the variation space for nonlinear operators between Hilbert spaces, and establishes its basic structural properties.
Section~\ref{section 3} develops the approximation theory for encoder--decoder two-layer neural networks.
Section~\ref{Sec: Generalization} establishes high-probability generalization bounds for path-norm-constrained empirical least squares under sub-Gaussian assumptions on the input and noise.
Sections~\ref{Sec: proof 1} and~\ref{Sec: proof 2} provide the proofs of the main approximation and generalization theorems, respectively.
The appendices collect auxiliary probabilistic results and deferred proofs used in the main text.

\section{Encoder--Decoder Architectures and Variation Spaces} \label{Sec: Encoder--Decoder Architectures and Variation Spaces}

Let \(\mathcal U\) and \(\mathcal V\) be separable real Hilbert spaces, equipped
with inner products \(\langle\cdot,\cdot\rangle_{\mathcal U}\) and
\(\langle\cdot,\cdot\rangle_{\mathcal V}\), and with the corresponding norms
\(\|\cdot\|_{\mathcal U}\) and \(\|\cdot\|_{\mathcal V}\), respectively.
Let \(\rho\) be a Borel probability measure on \(\mathcal U\), and let
\(\mathscr G^\dagger:\mathcal U\to\mathcal V\) be the target operator to be
approximated. We assume that \(\mathscr G^\dagger\) is strongly measurable
\cite[Chapter 1]{hytonen2016analysis}.
For \(1\le q<\infty\), we denote by $L^q_\rho(\mathcal V)
:=
L^q(\mathcal U,\rho;\mathcal V)$ the Lebesgue-Bochner space of (equivalence classes of) strongly measurable maps
\(\mathscr G:\mathcal U\to\mathcal V\), such that
\[
\|\mathscr G\|_{L^q_\rho(\mathcal V)}
:=
\left(
\int_{\mathcal U}
\|\mathscr G(u)\|_{\mathcal V}^q\,\mathrm d\rho(u)
\right)^{1/q}
<\infty.
\]
When \(q=2\), \(L^2_\rho(\mathcal V)\) is a Hilbert space endowed with the
inner product
\[
\langle \mathscr G_1,\mathscr G_2\rangle_{L^2_\rho(\mathcal V)}
:=
\int_{\mathcal U}
\langle \mathscr G_1(u),\mathscr G_2(u)\rangle_{\mathcal V}\,
\mathrm d\rho(u).
\]

Throughout the paper, $|\cdot|$ denotes the Euclidean norm. For a bounded
linear operator $T$ between normed spaces, $\|T\|$ denotes its operator norm.
For a Banach space $\mathcal X$, $I_{\mathcal X}$ denotes the identity
operator on $\mathcal X$, and we simply write $I$ when the underlying space is
clear.

\subsection{Encoder--Decoder Architectures}

We use orthonormal bases to define finite-dimensional encoders and decoders
for the input and output spaces. Let \(\Phi=\{\phi_i\}_{i\ge1}\) and
\(\Psi=\{\psi_j\}_{j\ge1}\) be orthonormal bases of \(\mathcal U\) and
\(\mathcal V\), respectively. For \(m\ge1\), define
\[
E_m^{\mathcal U}u
:=
\bigl(
\langle u,\phi_1\rangle_{\mathcal U},
\dots,
\langle u,\phi_m\rangle_{\mathcal U}
\bigr),
\qquad
D_m^{\mathcal U}x
:=
\sum_{i=1}^m x_i\phi_i ,
\]
where \(E_m^{\mathcal U}:\mathcal U\to\mathbb R^m\) is the input encoder and
\(D_m^{\mathcal U}:\mathbb R^m\to\mathcal U\) is the input decoder. We denote
the associated finite-dimensional input subspace and orthogonal projection by
\[
\mathcal U_m:=\operatorname{span}\{\phi_1,\dots,\phi_m\},
\qquad
P_m^{\mathcal U}:=D_m^{\mathcal U}E_m^{\mathcal U}.
\]
On the output side, for \(n\ge1\), set
\[
E_n^{\mathcal V}v
:=
\bigl(
\langle v,\psi_1\rangle_{\mathcal V},
\dots,
\langle v,\psi_n\rangle_{\mathcal V}
\bigr),
\qquad
D_n^{\mathcal V}y
:=
\sum_{j=1}^n y_j\psi_j ,
\]
where \(E_n^{\mathcal V}:\mathcal V\to\mathbb R^n\) and
\(D_n^{\mathcal V}:\mathbb R^n\to\mathcal V\) are the output encoder and
decoder. The corresponding output subspace and projection are
\[
\mathcal V_n:=\operatorname{span}\{\psi_1,\dots,\psi_n\},
\qquad
P_n^{\mathcal V}:=D_n^{\mathcal V}E_n^{\mathcal V}.
\]

With these encoders and decoders, we approximate the target operator by
composing a finite-dimensional map \(g:\mathbb R^m\to\mathbb R^n\) with the
input encoder and the output decoder:
\begin{equation}\label{encoder-decoder}
\mathscr G
=
D_n^{\mathcal V}\circ g\circ E_m^{\mathcal U}
:\mathcal U\to\mathcal V .
\end{equation}
In this representation, the input is described by its first \(m\) coordinates,
the map \(g\) acts on the encoded variables, and the output is reconstructed
in the finite-dimensional space \(\mathcal V_n\). We refer to
\eqref{encoder-decoder} as an encoder--decoder architecture.
The orthonormal bases may be specified a priori, as in Fourier, wavelet, or
polynomial bases, or constructed from data, as in PCA bases. Throughout this paper, the finite-dimensional map \(g\) is taken to be a
two-layer neural network.

The finite-dimensional representation introduces two natural projection
residuals. These quantities will enter the error bound. For
\(1\le q<\infty\), the input encoding error is defined by
\begin{equation} \label{input encoding error}
    \bigl\|I-P_m^{\mathcal U}\bigr\|_{L^q_\rho(\mathcal U)}
=
\left(
\int_{\mathcal U}
\bigl\|(I-P_m^{\mathcal U})u\bigr\|_{\mathcal U}^q
\,\mathrm d\rho(u)
\right)^{1/q}.
\end{equation}
If \(\mathscr G^\dagger\in L^q_\rho(\mathcal V)\), the output encoding error
is defined by
\begin{equation} \label{output encoding error}
    \bigl\|(I-P_n^{\mathcal V})\mathscr G^\dagger\bigr\|_{L^q_\rho(\mathcal V)}
=
\left(
\int_{\mathcal U}
\bigl\|(I-P_n^{\mathcal V})\mathscr G^\dagger(u)\bigr\|_{\mathcal V}^q
\,\mathrm d\rho(u)
\right)^{1/q}.
\end{equation}
These quantities are the \(L^q_\rho\)-norms of the input and output
projection residuals associated with \(\mathcal U_m\) and \(\mathcal V_n\),
respectively.

\subsection{Variation Space for Operator Learning}
\label{variation space}

In this subsection, we introduce a Hilbert-valued variation space as an infinite-dimensional
regularity class for target operators
\(\mathscr G^\dagger:\mathcal U\to\mathcal V\). This construction extends
the classical scalar-valued variation space to \(\mathcal V\)-valued
operators by encoding the output weights through finite
\(\mathcal V\)-valued measures. 
The resulting class is formulated directly
on the Hilbert spaces \(\mathcal U\) and \(\mathcal V\), and is independent
of the encoder--decoder dimensions as well as of the particular orthonormal
bases used to form finite representations.


Let \(\sigma:\mathbb R\to\mathbb R\) be a continuous activation
function. To incorporate the bias term in the two-layer representation into the Hilbert-space formulation, we
augment the input space by setting
\[
\overline{\mathcal U}:=\mathbb R\oplus \mathcal U,
\qquad
\bar u:=(1,u)\in \overline{\mathcal U}.
\]
For \(\bar h=(b,h)\in\overline{\mathcal U}\), this gives
\[
\langle \bar u,\bar h\rangle_{\overline{\mathcal U}}
:=
b+\langle u,h\rangle_{\mathcal U}.
\]
We endow the closed unit ball \(B_{\overline{\mathcal U}}\) with the weak
topology\footnote{Under the Riesz identification \(\overline{\mathcal U}\simeq \overline{\mathcal U}^*\), the weak topology on \(\overline{\mathcal U}\) corresponds to the weak-* topology on \(\overline{\mathcal U}^*\).}. Since \(\overline{\mathcal U}\) is separable, this ball is compact
and metrizable in the weak topology \cite{brezis2011functional}. Let $\mathcal M\bigl(B_{\overline{\mathcal U}};\mathcal V\bigr)$
denote the space of finite regular countably additive
\(\mathcal V\)-valued Borel measures of bounded variation on
\(B_{\overline{\mathcal U}}\). For
\(\Lambda\in\mathcal M(B_{\overline{\mathcal U}};\mathcal V)\),
we denote its total variation measure by \(|\Lambda|\) and set $\|\Lambda\|_{\mathrm{TV}}
:=
|\Lambda|(B_{\overline{\mathcal U}}).$ Then
\(\mathcal M(B_{\overline{\mathcal U}};\mathcal V)\)
is a Banach space under the total variation norm
\cite{1977Vector}.

We now introduce the variation space used in this paper.
\begin{definition}[Variation space] \label{def: variation space}
The variation space \(\mathbb V(\mathcal U;\mathcal V)\) consists of all maps
\(\mathscr G:\mathcal U\to\mathcal V\) that admit a representation of the form
\[
\mathscr G(u)
=
\int_{B_{\overline{\mathcal U}}}
\sigma\bigl(
\langle \bar u,\bar h\rangle_{\overline{\mathcal U}}
\bigr)\,
\mathrm d\Lambda(\bar h),
\qquad u\in\mathcal U,
\]
for some $\Lambda\in \mathcal M\bigl(B_{\overline{\mathcal U}};\mathcal V\bigr)$. The associated variation norm is
\[
\|\mathscr G\|_{\mathbb V(\mathcal U;\mathcal V)}
:=
\inf_{\Lambda}
|\Lambda|\bigl(B_{\overline{\mathcal U}}\bigr),
\]
where the infimum is taken over all finite \(\mathcal V\)-valued Borel
measures \(\Lambda\) representing \(\mathscr G\) in the above sense.
\end{definition}

The following proposition records basic properties of the variation
space.
\begin{proposition} \label{prop: property}
The following statements hold.
\begin{enumerate}[(1)]
    \item The variation space
    $\mathbb V(\mathcal U;\mathcal V)$, equipped with the variation norm
    $\|\cdot\|_{\mathbb V(\mathcal U;\mathcal V)}$, is a Banach space.

    \item For every $\mathscr G_1,\mathscr G_2
    \in\mathbb V(\mathcal U;\mathcal V)$ and $u\in\mathcal U$,
    \[
    \|\mathscr G_1(u)-\mathscr G_2(u)\|_{\mathcal V}
    \le
    \sup_{|t|\le \|\bar u\|_{\overline{\mathcal U}}}
    |\sigma(t)|
    \,
    \|\mathscr G_1-\mathscr G_2\|_{\mathbb V(\mathcal U;\mathcal V)}.
    \]
    In particular, convergence in
    $\mathbb V(\mathcal U;\mathcal V)$ implies pointwise convergence in
    $\mathcal V$.

    \item The infimum in Definition~\ref{def: variation space} is attained.
    More precisely, for every
    $\mathscr G\in\mathbb V(\mathcal U;\mathcal V)$, there exists
    $\Lambda_{\mathscr G}\in
    \mathcal M(B_{\overline{\mathcal U}};\mathcal V)$ representing
    $\mathscr G$ such that
    \[
    |\Lambda_{\mathscr G}|(B_{\overline{\mathcal U}})
    =
    \|\mathscr G\|_{\mathbb V(\mathcal U;\mathcal V)}.
    \]
\end{enumerate}
\end{proposition}



The data distribution $\rho$, the encoder--decoder dimensions $m$ and $n$,
and the chosen finite-dimensional bases do not enter the definition of
$\mathbb V(\mathcal U;\mathcal V)$. This separation allows us to formulate
regularity at the operator level and subsequently derive approximation and
learning bounds for finite-dimensional encoder--decoder neural networks.


The definition above is formulated in terms of
\(\mathcal V\)-valued measures. The next proposition provides a geometric
characterization of its unit ball through a
\(\mathcal V\)-valued single-neuron dictionary. Such a characterization
holds in any Bochner space \(L^p_\gamma(\mathcal V)\) satisfying the stated 
integrability condition.

\begin{proposition}[Characterization of the variation space]
\label{prop: characterization}
Let \(1\le p<\infty\), and let \(\gamma\) be a Borel probability
measure on \(\mathcal U\) satisfying
\[
\int_{\mathcal U}
\left(
\sup_{|t|\le \|\bar u\|_{\overline{\mathcal U}}}
|\sigma(t)|
\right)^p
\,\mathrm d\gamma(u)
<\infty.
\]
Let $S_{\mathcal V}
:=
\{v\in\mathcal V:\|v\|_{\mathcal V}=1\}$ denote the unit sphere of \(\mathcal V\). We define the
\(\mathcal V\)-valued dictionary associated with \(\sigma\) by
\[
\mathcal D_\sigma
:=
\left\{
u\mapsto
\sigma\bigl(
\langle \bar u,\bar h\rangle_{\overline{\mathcal U}}
\bigr)v
:\;
\bar h\in B_{\overline{\mathcal U}},
\quad
v\in S_{\mathcal V}
\right\}.
\]
Then, viewing each $\mathscr{G}\in\mathcal V(\mathcal U;\mathcal V)$ as its equivalence class in
$L_\gamma^p(\mathcal V)$, the image of the unit ball of
\(\mathbb V(\mathcal U;\mathcal V)\) is the closed convex hull of
\(\mathcal D_\sigma\) in \(L^p_\gamma(\mathcal V)\), namely,
\[
\left\{
\mathscr G\in\mathbb V(\mathcal U;\mathcal V):
\|\mathscr G\|_{\mathbb V(\mathcal U;\mathcal V)}
\le 1
\right\}
=
\overline{\operatorname{conv}(\mathcal D_\sigma)}
^{\,L^p_\gamma(\mathcal V)}.
\]
\end{proposition}

Proposition~\ref{prop: characterization} thus connects the
measure representation in Definition~\ref{def: variation space}
with the convex geometry of the associated single-neuron dictionary. 
We next identify the linear operators contained in the variation space for
the ReLU activation.
Let $\mathcal S_1(\mathcal U,\mathcal V)$ denote the
Schatten-$1$ class of compact operators $T:\mathcal U\to\mathcal V$ with
summable singular values, equipped with the norm given by the sum of its
singular values.

\begin{proposition}[Schatten-$1$ characterization of linear operators] 
\label{prop: linear characterization}
Suppose that $\sigma(t)=\operatorname{ReLU}(t)=t_+$. Let
$T:\mathcal U\to\mathcal V$ be a bounded linear operator. Then
\[
T\in\mathbb V(\mathcal U;\mathcal V)
\quad\Longleftrightarrow\quad
T\in\mathcal S_1(\mathcal U,\mathcal V).
\]
Moreover,
\[
\|T\|_{\mathbb V(\mathcal U;\mathcal V)}
=
2\|T\|_{\mathcal S_1(\mathcal U,\mathcal V)}.
\]
In particular, when $\mathcal U=\mathcal V$, the bounded linear operators in
$\mathbb V(\mathcal U;\mathcal U)$ are precisely the trace-class operators.
\end{proposition}

Proposition~\ref{prop: linear characterization} shows that, for linear
operators, the ReLU variation norm coincides up to the factor $2$ with the
Schatten-$1$ norm. Thus $\mathbb V(\mathcal U;\mathcal V)$ provides a
nonlinear extension of the Schatten-$1$ structure.

The proofs of Propositions~\ref{prop: property},
\ref{prop: characterization}, and
\ref{prop: linear characterization} are deferred to
Appendices~\ref{Proof of prop: property},
\ref{Proof of prop: characterization}, and
\ref{Proof of prop: linear characterization}, respectively.

\noindent\textbf{A Hammerstein-type operator.}
For the ReLU activation, the variation space contains a natural class of
nonlinear integral operators of Hammerstein type. Let $(X,\mu)$ be a measure
space, and suppose that $\mathcal U$ is a reproducing kernel Hilbert space on
$X$. For each $x\in X$, let $k_x\in\mathcal U$ be the representer of point
evaluation, so that $u(x)=\langle u,k_x\rangle_{\mathcal U}$. Let
$b:X\to\mathbb R$ be measurable, and suppose that the maps $x\mapsto k_x$
and $x\mapsto v_x$ are strongly measurable with values in $\mathcal U$ and
$\mathcal V$, respectively. Set
\[
r(x):=\left(b(x)^2+\|k_x\|_{\mathcal U}^2\right)^{1/2},
\qquad
\int_X r(x)\|v_x\|_{\mathcal V}\,d\mu(x)<\infty.
\]
Then
\[
\mathscr G(u):=\int_X
\operatorname{ReLU}\bigl(b(x)+u(x)\bigr)v_x\,d\mu(x),
\qquad u\in\mathcal U,
\]
defines an operator $\mathscr G:\mathcal U\to\mathcal V$ belonging to
$\mathbb V(\mathcal U;\mathcal V)$, with
\[
\|\mathscr G\|_{\mathbb V(\mathcal U;\mathcal V)}
\leq
\int_X r(x)\|v_x\|_{\mathcal V}\,d\mu(x).
\]
Indeed, let $\theta(x)=r(x)^{-1}(b(x),k_x)$ when $r(x)>0$, and let
$\theta(x)=0$ otherwise. Then
$\theta:X\to B_{\overline{\mathcal U}}$ is measurable. By the positive
homogeneity of ReLU, the pushforward under $\theta$ of the finite
$\mathcal V$-valued measure $r(x)v_x\,d\mu(x)$ represents $\mathscr G$. In particular, when $\mathcal V$ is a function space and
$v_x=K(\cdot,x)$, this includes integral operators of the form
\[
\mathscr G(u)(y)=\int_X K(y,x)\operatorname{ReLU}(b(x)+u(x))\,d\mu(x).
\]


\noindent\textbf{Relation to integral vector-valued RKBSs.}
The representation in Definition~\ref{def: variation space} falls within the
integral vector-valued RKBS framework~\cite{BartolucciEtAl2024,DummerEtAl2025}.
In the setting of operator learning, we establish the exact
Schatten-$1$ characterization for ReLU, together with the encoder--decoder
approximation and generalization results developed below.


\noindent\textbf{Comparison with Banach-valued variation spaces.}
A closely related work is \cite{korolev2022two}, where vector-valued
variation spaces are developed for two-layer neural networks between Banach
spaces. In that framework, the output space \(\mathcal V\) is assumed to be a
Riesz space, and the ReLU activation is interpreted as the lattice operation
of taking the positive part. The resulting approximating networks have the
form
\[
f_N(u)=\sum_{i=1}^N \alpha_i (K_i u)_+,
\]
where $\alpha_i\in\mathbb{R}$, \(K_i:\mathcal U\to\mathcal V\) are finite-rank operators and the
positive part is taken in the output space \(\mathcal V\). Thus the
nonlinearity is imposed through the lattice structure of \(\mathcal V\), rather
than through a finite-dimensional neural network acting on encoded
coordinates.

This lattice-based formulation requires the lattice operations in
\(\mathcal V\) to be sequentially continuous in the weak-* topology, which
is restrictive in infinite dimensions. For example, this condition fails for
\(\mathcal V=L^q([0,1]^d)\), \(1<q\le\infty\), while it holds for sequence
spaces \(\mathcal V=\ell^q\), \(1<q\le\infty\), and for certain Lipschitz
spaces. 

Another important difference concerns the error topology and the role of
finite-dimensional representations. 
The approximation results in \cite{korolev2022two} are formulated in
Bochner spaces defined with respect to a metric \(d^*\) that metrizes
the weak-* topology on bounded subsets of the output space
\(\mathcal V\). Our approximation error is instead measured in the norm of the Bochner space \(L^q(\mathcal U,\rho;\mathcal V)\). Moreover, although the
approximating networks there involve finite-rank operators
\(K_i:\mathcal U\to\mathcal V\), these operators are not tied to prescribed
input encoders and output decoders of the form
\(D_n^{\mathcal V}\circ g\circ E_m^{\mathcal U}\), and no separate finite-resolution errors are
quantified. 
In contrast, our construction does not require a lattice structure on \(\mathcal V\) and yields encoder--decoder approximation and generalization bounds with explicit decompositions into input-encoding, output-encoding, and finite-width terms, together with a finite-sample term in the generalization bound.

\section{Operator Approximation Theory in Variation Space}
\label{section 3}

Fix \(2\le q<\infty\). Throughout this section, we assume that
\(\sigma:\mathbb R\to\mathbb R\) is Lipschitz continuous with Lipschitz
constant \(L_\sigma\), and that \(\rho\) is a Borel probability measure
on \(\mathcal U\) with finite \(q\)-th moment, namely,
\[
M_q
:=
\left(
\int_{\mathcal U}\|u\|_{\mathcal U}^q\,\mathrm d\rho(u)
\right)^{1/q}
<\infty.
\]
We also write
\[
\bar M_q
:=
\left(
\int_{\mathcal U}
(1+\|u\|_{\mathcal U}^2)^{q/2}\,\mathrm d\rho(u)
\right)^{1/q}.
\]
For each \(m\ge1\), let $\mu_m:=(E_m^{\mathcal U})_\#\rho$ be the distribution of the encoded input \(E_m^{\mathcal U}u\).

We next specify the finite-dimensional network class used for the map acting
between the encoded input and output coordinate spaces. For \(B>0\) and \(m,n,N\ge1\), define
\begin{equation} \label{NN class}
\mathscr G_{m,n,N}(B)
:=
\left\{
g:\mathbb R^m\to\mathbb R^n
\;\middle|\;
\begin{array}{@{}l@{}}
\displaystyle
g(x)=\sum_{k=1}^N b_k\,\sigma(\alpha_k+a_k\cdot x),\\[0.5em]
\displaystyle
\max_{1\le k\le N}
\bigl(\alpha_k^2+|a_k|^2\bigr)\le 1,\\[0.5em]
\displaystyle
\sum_{k=1}^N |b_k|\le B
\end{array}
\right\},
\end{equation}
where \(\alpha_k\in\mathbb R\), \(a_k\in\mathbb R^m\),
\(b_k\in\mathbb R^n\). 

Under the finite \(q\)-th moment assumption on \(\rho\), every
\(g\in\mathscr G_{m,n,N}(B)\) belongs to
\(L^q(\mathbb R^m,\mu_m;\mathbb R^n)\). This follows from the Lipschitz
continuity of \(\sigma\) and the normalization
\(\alpha_k^2+|a_k|^2\le1\), which give at most linear growth in \(|x|\), and
from \(|E_m^{\mathcal U}u|\le \|u\|_{\mathcal U}\).

We next establish a finite-width upper bound for the encoder--decoder
architecture. The finite-dimensional component is a multi-output two-layer
network in \(\mathscr G_{m,n,N}(B)\), where \(B\) is controlled by the
variation norm of the target operator. The estimate separates three
contributions: the input encoding error associated with
\(P_m^{\mathcal U}\), the output encoding error associated with
\(P_n^{\mathcal V}\), and the finite-width approximation error of the
finite-dimensional two-layer network.

\begin{theorem}[Finite-width encoder--decoder approximation]
\label{thm:finite-width-approximation}
Let \(2\leq q<\infty\), and suppose that \(\rho\) is a Borel probability
measure on \(\mathcal U\) with finite \(q\)-th moment. Then there exists a
constant \(C_q>0\), depending only on \(q\), such that for every
\(\mathscr G^\dagger\in\mathbb V(\mathcal U;\mathcal V)\), and every
\(m,n,N\geq1\), there exists
\[
g_{m,n,N}
\in
\mathscr G_{m,n,N}
\left(
\|\mathscr G^\dagger\|_{\mathbb V(\mathcal U;\mathcal V)}
\right)
\]
satisfying
\[
\begin{aligned}
\bigl\|\mathscr G^\dagger-D_n^{\mathcal V}\circ g_{m,n,N}\circ E_m^{\mathcal U}\bigr\|_{L^q_\rho(\mathcal V)}
&\le
\underbrace{L_\sigma
\|\mathscr G^\dagger\|_{\mathbb V(\mathcal U;\mathcal V)}
\bigl\|I-P_m^{\mathcal U}\bigr\|_{L^q_\rho(\mathcal U)}}_{\text{input-encoding term}} +  \underbrace{\bigl\|(I-P_n^{\mathcal V})\mathscr G^\dagger
\bigr\|_{L^q_\rho(\mathcal V)}}_{\text{output-encoding term}}
\\[0.5em]
&\quad+
\underbrace{\frac{
C_q\bigl(|\sigma(0)|+L_\sigma \bar M_q\bigr)\|\mathscr G^\dagger\|_{\mathbb V(\mathcal U;\mathcal V)}
}{\sqrt N}}_{\text{finite-width term}}.
\end{aligned}
\]
\end{theorem}

The estimate in Theorem~\ref{thm:finite-width-approximation} separates three distinct sources of approximation error. The input-encoding term quantifies the information loss incurred by replacing \(u\) with its finite-dimensional projection \(P_m^{\mathcal U}u\), with the factor \(L_\sigma\|\mathscr G^\dagger\|_{\mathbb V(\mathcal U;\mathcal V)}\) reflecting the Lipschitz control provided by the variation norm. The output-encoding term measures the projection error arising from restricting the reconstructed output to the decoder space \(\mathcal V_n\). The finite-width term quantifies the additional error incurred by approximating the variation space representation with a two-layer network of \(N\) neurons. Crucially, it decays as \(N^{-1/2}\), with a constant independent of the encoding dimensions \(m\) and \(n\) and of the particular orthonormal bases used to define the encoders and decoders. Thus, the dependence on the finite-dimensional representation is confined to the two encoding terms.

\noindent\textbf{Parametric complexity.} Suppose, in addition, that the input and output encoding errors decay algebraically with the encoding dimensions, namely,
\begin{equation} \label{encoding errors decay algebraically}
    \bigl\|I-P_m^{\mathcal U}\bigr\|_{L^q_\rho(\mathcal U)}
\lesssim m^{-s_{\mathcal U}},
\qquad
\bigl\|(I-P_n^{\mathcal V})\mathscr G^\dagger\bigr\|_{L^q_\rho(\mathcal V)}
\lesssim n^{-s_{\mathcal V}}
\end{equation}
for some \(s_{\mathcal U},s_{\mathcal V}>0\). Then Theorem~\ref{thm:finite-width-approximation} yields
\[
\bigl\|\mathscr G^\dagger
-D_n^{\mathcal V}\circ g_{m,n,N}\circ E_m^{\mathcal U}
\bigr\|_{L^q_\rho(\mathcal V)}
\lesssim
m^{-s_{\mathcal U}}+n^{-s_{\mathcal V}}+N^{-1/2},
\]
where the implicit constant is independent of \(m\), \(n\), and \(N\). Hence an accuracy of order \(\varepsilon\) can be achieved by choosing
\[
m\asymp \varepsilon^{-1/s_{\mathcal U}},
\qquad
n\asymp \varepsilon^{-1/s_{\mathcal V}},
\qquad
N\asymp \varepsilon^{-2}.
\]
In particular, the required number of neurons grows algebraically in \(\varepsilon^{-1}\). When the encoder and decoder bases are fixed, a network in \(\mathscr G_{m,n,N}\) has \(N(m+n+1)\) trainable scalar parameters, and hence the corresponding parametric complexity satisfies
\[
N(m+n+1)
\lesssim
\varepsilon^{-2-1/s_{\mathcal U}}
+
\varepsilon^{-2-1/s_{\mathcal V}}
+
\varepsilon^{-2}
\lesssim
\varepsilon^{-2-\max\{1/s_{\mathcal U},\,1/s_{\mathcal V}\}}.
\]
Consequently, for classes of target operators with uniformly bounded variation norms and algebraic decay of the input and output encoding errors, the approximation complexity is algebraic in \(\varepsilon^{-1}\), thereby avoiding the curse of parametric complexity \cite{lanthaler2023operator,cheng2026learning,lanthaler2026parametric} encountered for much broader classes of nonlinear operators.

\noindent\textbf{Approximation beyond the variation space.} Theorem~\ref{thm:finite-width-approximation} assumes that the target operator belongs to $\mathbb V(\mathcal U;\mathcal V)$. More generally, approximation rates can be obtained for operators that are well approximated by bounded subsets of the variation space. For $\mathscr G^\dagger\in L^q_\rho(\mathcal V)$ and $R>0$, define \[ \mathcal E_R(\mathscr G^\dagger) := \inf_{\substack{ \mathscr G\in\mathbb V(\mathcal U;\mathcal V)\\ \|\mathscr G\|_{\mathbb V(\mathcal U;\mathcal V)}\le R }} \|\mathscr G^\dagger-\mathscr G\|_{L^q_\rho(\mathcal V)}. \] 
\begin{corollary}[Approximation for general target operators] \label{cor:general-target-approximation} Let $2\le q<\infty$, and suppose that $\rho$ is a Borel probability measure on $\mathcal U$ with finite $q$-th moment. Then, for every $\mathscr G^\dagger\in L^q_\rho(\mathcal V)$, $R>0$, and $m,n,N\ge1$, 
\[ \begin{aligned} 
\inf_{g\in\mathscr G_{m,n,N}(R)} \bigl\| \mathscr G^\dagger - D_n^{\mathcal V}\circ g\circ E_m^{\mathcal U} \bigr\|_{L^q_\rho(\mathcal V)} 
\;\le\;& 2\mathcal E_R(\mathscr G^\dagger) \;+\; L_\sigma R \bigl\|I-P_m^{\mathcal U}\bigr\|_{L^q_\rho(\mathcal U)} 
\\&\;+\; \bigl\| (I-P_n^{\mathcal V})\mathscr G^\dagger \bigr\|_{L^q_\rho(\mathcal V)} \\& \;+\; \frac{ C_q\bigl(|\sigma(0)|+L_\sigma\bar M_q\bigr)R }{\sqrt N}, \end{aligned} 
\] 
where $C_q>0$ is the constant in Theorem~\ref{thm:finite-width-approximation}. \end{corollary}

\begin{proof} Let $\varepsilon>0$, and choose $\mathscr G_R\in\mathbb V(\mathcal U;\mathcal V)$ with $\|\mathscr G_R\|_{\mathbb V(\mathcal U;\mathcal V)}\le R$ such that \[ \|\mathscr G^\dagger-\mathscr G_R\|_{L^q_\rho(\mathcal V)} \le \mathcal E_R(\mathscr G^\dagger)+\varepsilon. \] 
Applying Theorem~\ref{thm:finite-width-approximation} to $\mathscr G_R$,
and using the triangle inequality to bound
$\|(I-P_n^{\mathcal V})\mathscr G_R\|_{L^q_\rho(\mathcal V)}$
by
$\|(I-P_n^{\mathcal V})\mathscr G^\dagger\|_{L^q_\rho(\mathcal V)}
+\|\mathscr G_R-\mathscr G^\dagger\|_{L^q_\rho(\mathcal V)}$,
the claimed estimate follows with
$2(\mathcal E_R(\mathscr G^\dagger)+\varepsilon)$ in place of
$2\mathcal E_R(\mathscr G^\dagger)$. Letting
$\varepsilon\downarrow0$ completes the proof.
\end{proof}

\noindent\textbf{Comparison with finite-dimensional variation-space results.} In the finite-dimensional multi-output setting, Shenouda et al.~\cite{shenouda2024variation} introduced variation spaces for functions from \(\mathbb R^d\) to \(\mathbb R^D\). Their Theorem~3 gives a \(K^{-1/2}\) approximation rate on \(Q^d=[0,1]^d\), with the bound
\[
\|f-f_K\|_{L^2(Q^d;\mathbb R^D)}
\lesssim
D^{3/2}
\|f\|_{\mathcal V_\sigma(Q^d;\mathbb R^D)}
K^{-1/2}.
\]
Thus, the bound contains the explicit factor \(D^{3/2}\) depending on the output dimension. In contrast, when the encoding errors vanish, Theorem~\ref{thm:finite-width-approximation} yields the same \(N^{-1/2}\) finite-width rate with a constant independent of the input and output encoding dimensions and of the encoder--decoder bases.

Sharper rates are available in fixed finite dimensions when additional geometric structure is exploited. For scalar-valued variation spaces associated with the power-ReLU activation $\operatorname{ReLU}^k(t):=(t_+)^k$ on the unit ball in $\mathbb R^d$, equipped with the $L^2$ norm with respect to Lebesgue measure, Siegel and Xu~\cite{siegel2024sharp} established the sharp approximation rate $N^{-\frac12-\frac{2k+1}{2d}}$, which reduces to $N^{-\frac12-\frac{3}{2d}}$ for ReLU. Weighted variation spaces were subsequently introduced in~\cite{devore2025weighted}, enlarging the classical ReLU variation space while preserving the same $N^{-\frac12-\frac{3}{2d}}$ rate in the Lebesgue $L^2$ norm on bounded Euclidean domains. These improved rates exploit both the finite-dimensional geometry and the smooth parametrization of ReLU ridge functions, and do not directly extend to the general pushforward measures $\mu_m=(E_m^{\mathcal U})_\#\rho$ arising in our setting. 
Moreover, the sharp ReLU exponent satisfies
\(\frac{1}{2}+\frac{3}{2d}\to\frac{1}{2}\) as \(d\to\infty\).
Although this comparison does not establish optimality in the present
setting, it suggests that the \(N^{-1/2}\) exponent is a natural
dimension-uniform benchmark.

\noindent\textbf{Remarks on extensions.}
The preceding arguments extend in two useful directions.

\smallskip
\noindent\emph{Polynomially growing activations.}
For the approximation results, global Lipschitz continuity may be replaced by
suitable polynomial growth conditions. More precisely, suppose that, for some
\(r\geq1\),
\[
|\sigma(t)|\leq C_\sigma(1+|t|^r),
\qquad
|\sigma(t)-\sigma(s)|
\leq
C_\sigma |t-s|(1+|t|+|s|)^{r-1}.
\]
If \(\rho\) has a finite \(rq\)-th moment, the sampling argument still yields
a finite-width term of order \(N^{-1/2}\), with a constant independent of the
encoding dimensions, while the input-encoding term is replaced by
\(\bigl\|(1+\|u\|_{\mathcal U})^{r-1}
\|(I-P_m^{\mathcal U})u\|_{\mathcal U}\bigr\|_{L^q_\rho}\).
By H\"older's inequality, this term is controlled by the \(rq\)-th moment of
\(\rho\) and
\(\|I-P_m^{\mathcal U}\|_{L^{rq}_\rho(\mathcal U)}\).
In particular, this yields the corresponding approximation result for
\(\operatorname{ReLU}^k\).

\smallskip
\noindent\emph{Uniformly stable linear encoders and decoders.}
The orthonormal encoder--decoder construction may be replaced by linear maps
satisfying
\[
\sup_{m\geq1}
\bigl(
\|E_m^{\mathcal U}\|+\|D_m^{\mathcal U}\|
\bigr)
+
\sup_{n\geq1}
\bigl(
\|E_n^{\mathcal V}\|+\|D_n^{\mathcal V}\|
\bigr)
<\infty,
\]
together with
\[
\bigl\|(I-D_m^{\mathcal U}E_m^{\mathcal U})u
\bigr\|_{L^q_\rho(\mathcal U)}
\to0
\quad (m\to\infty),
\qquad
\bigl\|(I-D_n^{\mathcal V}E_n^{\mathcal V})
\mathscr G^\dagger
\bigr\|_{L^q_\rho(\mathcal V)}
\to0
\quad (n\to\infty).
\]
The approximation argument then gives an analogous estimate with
\(P_m^{\mathcal U}\) and \(P_n^{\mathcal V}\) replaced by
\(D_m^{\mathcal U}E_m^{\mathcal U}\) and
\(D_n^{\mathcal V}E_n^{\mathcal V}\), respectively. \footnote{The approximation
estimate only requires the triangle inequality. In the squared-error
analysis, the Pythagorean identity may be replaced by
\(\|a+b\|_{\mathcal V}^2
\leq 2\|a\|_{\mathcal V}^2+2\|b\|_{\mathcal V}^2\),
which changes only fixed multiplicative constants.} The unit bound on the hidden weights and the outer-weight radius are adjusted
by fixed factors depending only on the above stability bounds, so the
resulting constants remain independent of \(m\) and \(n\).


For \(q=2\), under the distributional and activation assumptions of
Section~\ref{Sec: Generalization}, the generalization argument can likewise
be adapted by working in the encoded output space. In the resulting squared
prediction-error bound, the finite-width and finite-sample terms remain of
orders \(N^{-1}\) and \(K^{-1/2}\), respectively. The constants may depend
additionally on the stability bounds but remain independent of \(m\) and
\(n\), and the finite-sample constant remains independent of \(N\); the
input- and output-encoding contributions are replaced by the corresponding
squared \(L^2_\rho\)-reconstruction errors.

\section{Generalization of Operator Learning in Variation Space}
\label{Sec: Generalization}

Let $(u,v)$ be a random pair taking values in $\mathcal U\times\mathcal V$, with $u\sim\rho$. We assume the regression model
\[
v=\mathscr G^\dagger(u)+\varepsilon,\qquad \mathbb E[\varepsilon\mid u]=0,
\]
where $\mathscr G^\dagger\in \mathbb V(\mathcal U;\mathcal V)$.
Let $(u_k,v_k)_{k=1}^K$ be independent copies of $(u,v)$.
Throughout this section, the encoder and decoder bases are assumed to be deterministic, or to be constructed from an auxiliary sample independent of the training sample \(\{(u_k,v_k)\}_{k=1}^K\). In the latter case, all probability statements are understood conditionally on the bases. We impose the following sub-Gaussian assumptions on the input and the noise.

\begin{assumption}[Sub-Gaussian input and noise] \label{assump}
There exist constants $C_u,C_\varepsilon>0$ such that for all $t\ge 0$,
\[
\mathbb P\left\{\sqrt{1+\left\|u\right\|_{\mathcal{U}}^{2}}>t\right\}
\le 2\exp\left(-\frac{t^2}{C_u^2}\right),
\]
and
\[
\mathbb P\left\{\|\varepsilon\|_{\mathcal V}>t\right\}
\le 2\exp\left(-\frac{t^2}{C_\varepsilon^2}\right).
\]
\end{assumption}
In particular, $(\bar M_2)^2=\mathbb E(1+\|u\|_{\mathcal U}^{2})\leq2C_u^2$.
Throughout this section, the activation function $\sigma:\mathbb{R}\to\mathbb{R}$ is assumed to be Lipschitz and positively homogeneous of degree one, namely
\[
\sigma(\alpha t)=\alpha \sigma(t), \qquad \forall t\in\mathbb{R},\ \forall \alpha\ge 0.
\]
Typical examples include the ReLU and leaky ReLU activations.






For \(m,n,N\geq 1\) and \(B>0\), let
\[
\Theta_{m,n,N}(B)
:=
\left\{
(\alpha_j,a_j,b_j)_{j=1}^N
\in
\bigl(\mathbb R\times\mathbb R^m\times\mathbb R^n\bigr)^N
\;\middle|\;
\begin{array}{@{}l@{}}
\displaystyle
\max_{1\le j\le N}
\bigl(\alpha_j^2+|a_j|^2\bigr)\le 1,\\[0.5em]
\displaystyle
\sum_{j=1}^N |b_j|\le B
\end{array}
\right\}
\]
be the parameter set associated with the finite-dimensional network class
\(\mathscr G_{m,n,N}(B)\) defined in \eqref{NN class}, with the hidden
parameters normalized by \(\alpha_j^2+|a_j|^2\leq 1\). For
\(\Theta\in\Theta_{m,n,N}(B)\), define
\[
g_\Theta:\mathbb R^m\to\mathbb R^n,
\qquad
g_\Theta(x)
:=
\sum_{j=1}^N b_j\sigma(\alpha_j+a_j\cdot x),
\]
and the induced encoder--decoder operator $\mathscr G_\Theta
:=
D_n^{\mathcal V}\circ g_\Theta\circ E_m^{\mathcal U}.$
Since \(\sigma\) is positively homogeneous in this section, this normalized
parametrization generates the same finite-width network class as the
unnormalized path-norm constraint
\[
\sum_{j=1}^N |b_j|\sqrt{\alpha_j^2+|a_j|^2}\leq B,
\]
with the scale of each hidden weight absorbed into the corresponding outer
weight $b_j$.

Given i.i.d. training samples \((u_k,v_k)_{k=1}^K\), set
\[
x_k:=E_m^{\mathcal U}u_k,
\qquad
y_k:=E_n^{\mathcal V}v_k,
\qquad
k=1,\ldots,K.
\]
We estimate the finite-dimensional map between the encoded input and
output spaces by empirical least squares over the path-norm-constrained
network class. Specifically, let $\widehat{\Theta}_B$ be any empirical risk minimizer, that is,
\[
\widehat\Theta_B
\in
\operatorname*{arg\,min}_{\Theta\in\Theta_{m,n,N}(B)}
\frac1K\sum_{k=1}^K |g_\Theta(x_k)-y_k|^2,
\qquad
\widehat{\mathscr G}_B:=\mathscr G_{\widehat\Theta_B}.
\]
Since \(D_n^{\mathcal V}\) is an isometry from \(\mathbb R^n\) onto \(\mathcal V_n\) and
\(D_n^{\mathcal V}y_k=P_n^{\mathcal V}v_k\), we have
\[
|g_\Theta(x_k)-y_k|
=
\|\mathscr G_\Theta(u_k)-P_n^{\mathcal V}v_k\|_{\mathcal V},
\qquad
k=1,\ldots,K.
\]
Thus the above minimization is equivalently the empirical least-squares problem
\[
\begin{aligned}
\min_{\Theta\in\Theta_{m,n,N}(B)}
\frac{1}{K}\sum_{k=1}^K
\bigl\|\mathscr G_\Theta(u_k)-v_k\bigr\|_{\mathcal V}^2
&=
\min_{\Theta\in\Theta_{m,n,N}(B)}
\frac{1}{K}\sum_{k=1}^K
\bigl\|\mathscr G_\Theta(u_k)-P_n^{\mathcal V}v_k\bigr\|_{\mathcal V}^2
\\
&\quad+
\frac{1}{K}\sum_{k=1}^K
\bigl\|(I-P_n^{\mathcal V})v_k\bigr\|_{\mathcal V}^2 ,
\end{aligned}
\]
that is, the least-squares problem for the induced encoder--decoder operators
with responses in $\mv$. The minimizer exists by compactness of
\(\Theta_{m,n,N}(B)\) and continuity of the empirical loss.

We are now ready to state the generalization bound for the empirical least-squares estimator.

\begin{theorem}[Generalization bound under a path-norm constraint]
\label{thm:generalization-path}
Suppose Assumption \ref{assump} holds. Let \(m,n,N,K\geq 1\), and let \(B>0\) satisfy $B\geq \|\mathscr G^\dagger\|_{\mathbb V(\mathcal U;\mathcal V)}$. Let $L_{K,\delta}:=\log\frac{12K}{\delta},$ and let \(C_2\) denote the constant \(C_q\) in Theorem \ref{thm:finite-width-approximation} for \(q=2\). Then, for every \(\delta\in(0,1)\), with probability at least \(1-\delta\), every empirical risk minimizer $\widehat{\Theta}_B$ satisfies
\[
\begin{aligned}
&\|\widehat{\mathscr G}_B-\mathscr G^\dagger\|_{L_\rho^2(\mathcal V)}^2
\\&\qquad\leq
\underbrace{2L_\sigma^2\|\mathscr G^\dagger\|_{\mathbb V(\mathcal U;\mathcal V)}^2
\|I-P_m^{\mathcal U}\|_{L_\rho^2(\mathcal U)}^2}_{\text{input-encoding term}}
\\&\qquad+ \underbrace{\|(I-P_n^{\mathcal V})\mathscr G^\dagger\|_{L_\rho^2(\mathcal V)}^2}_{\text{output-encoding term}}+ \underbrace{\frac{4C_2^2C_u^2L_\sigma^2\|\mathscr G^\dagger\|_{\mathbb V(\mathcal U;\mathcal V)}^2}{N}}_{\text{finite-width term}}
\\[0.5em]
&\qquad+
\underbrace{\left[
61L_\sigma^2(B+\|\mathscr G^\dagger\|_{\mathbb V(\mathcal U;\mathcal V)})^2C_u^2
+
66L_\sigma(B+\|\mathscr G^\dagger\|_{\mathbb V(\mathcal U;\mathcal V)})C_uC_\varepsilon
\right]
\frac{L_{K,\delta}^{3/2}}{\sqrt K}}_{\text{finite-sample term}}.
\end{aligned}
\]
\end{theorem}

Relative to the squared approximation bound implied by Theorem~\ref{thm:finite-width-approximation}, Theorem~\ref{thm:generalization-path} introduces an additional finite-sample term, which decays as \(K^{-1/2}\) up to a logarithmic factor. Crucially, its constant is independent of the encoding dimensions \(m\) and \(n\), the network width \(N\), and the chosen encoder--decoder bases. Consequently, for classes of target operators with uniformly bounded variation norms and suitably controlled encoding errors, the sample complexity is algebraic in the target accuracy. Under the algebraic encoding assumptions \eqref{encoding errors decay algebraically},
choosing $m\asymp K^{1/(4s_{\mathcal U})}$,
$n\asymp K^{1/(4s_{\mathcal V})}$, and $N\asymp K^{1/2}$
yields a squared prediction error of order $K^{-1/2}$, up to logarithmic factors.
This contrasts with the sample-complexity barriers for substantially broader classes of nonlinear operators, including general Lipschitz and Fr\'echet differentiable operator classes, as well as Gaussian Sobolev operator classes~\cite{adcock2024sample,adcock2025towards,cheng2026learning}. 

\noindent\textbf{Comparison with finite-dimensional generalization results.}
In the finite-dimensional scalar setting, \cite{e2019priori} studies two-layer networks with an $\ell_1$-type path norm, which in our notation takes the form $\sum_{j=1}^N |b_j|\bigl(|\alpha_j|+\|a_j\|_1\bigr),$ and target functions in the associated Barron space. Their population-risk bound decays as $K^{-1/2}$ up to logarithmic factors, with a $\sqrt{\log d}$ dependence arising from the $\ell_1$ geometry. In contrast, Theorem~\ref{thm:generalization-path} has the same \(K^{-1/2}\) sample exponent, while its finite-sample term is independent of the input and output encoding dimensions, the network width, and the encoder--decoder bases.

Sharper statistical rates are available for scalar-valued ReLU variation spaces in fixed finite dimensions. For the uniform distribution on the unit ball, \cite{parhi2022near} established, up to logarithmic factors, the minimax squared-error rate $K^{-\frac{d+3}{2d+3}}.$ For input distributions supported on the unit ball, \cite{yang2024nonparametric} showed that norm-constrained finite-width shallow ReLU networks attain the same exponent, up to a logarithmic factor, using a width-independent local Rademacher complexity bound. These sharper rates exploit the finite-dimensional structure of ReLU variation spaces, and the associated constants depend on the ambient dimension. In our setting, the encoded distributions
\(\mu_m=(E_m^{\mathcal U})_\#\rho\) may vary with \(m\) and are not required to satisfy a prescribed finite-dimensional support geometry. Theorem~\ref{thm:generalization-path} instead gives a bound uniform in the encoding dimensions under dimension-uniform tail assumptions. Since
$\frac{d+3}{2d+3}\to\frac12$ as $d\to\infty$, the gain in the fixed-dimensional exponent over the \(K^{-1/2}\) scaling vanishes as the dimension increases.

\section{Proof of Theorem \ref{thm:finite-width-approximation}} \label{Sec: proof 1}

Choose a representing measure
\(\Lambda\in\mathcal M(B_{\overline{\mathcal U}};\mathcal V)\) such that
\[
\mathscr G^\dagger(u)
=
\int_{B_{\overline{\mathcal U}}}
\sigma\bigl(\langle \bar u,\bar h\rangle_{\overline{\mathcal U}}\bigr)\,
\mathrm d\Lambda(\bar h),
\qquad u\in\mathcal U,
\]
and $V_\Lambda:=|\Lambda|(B_{\overline{\mathcal U}})
=
\|\mathscr G^\dagger\|_{\mathbb V(\mathcal U;\mathcal V)}.$ If \(V_\Lambda=0\), then \(\mathscr G^\dagger=0\), and the assertion follows
by taking \(g_{m,n,N}=0\). Henceforth, we assume
\(V_\Lambda>0\).

By the polar decomposition of vector measures (see, for example, \cite[Corollary 4]{1977Vector}, \cite[Theorem 8]{carmeli2010vector}, and \cite[Theorem 4.1 in Chapter VII]{lang2012real}), we write
\[
\mathrm d\Lambda(\bar h)=\lambda(\bar h)\,\mathrm d|\Lambda|(\bar h),
\qquad
\|\lambda(\bar h)\|_{\mathcal V}=1
\quad |\Lambda|\text{-a.e.}
\]
Thus
\[
\mathscr G^\dagger(u)
=
\int_{B_{\overline{\mathcal U}}}
\sigma\bigl(\langle \bar u,\bar h\rangle_{\overline{\mathcal U}}\bigr)\,
\lambda(\bar h)\,\mathrm d|\Lambda|(\bar h),
\qquad u\in\mathcal U.
\]

We first verify that \(\mathscr G^\dagger\in L^q_\rho(\mathcal V)\). Since
\(\sigma\) is globally Lipschitz,
\[
|\sigma(t)|\le |\sigma(0)|+L_\sigma |t|,
\qquad t\in\mathbb R.
\]
Therefore, for every \(u\in\mathcal U\), we have
\[
\begin{aligned}
\|\mathscr G^\dagger(u)\|_{\mathcal V}
&\le
\int_{B_{\overline{\mathcal U}}}
\bigl|\sigma(\langle \bar u,\bar h\rangle_{\overline{\mathcal U}})\bigr|
\,\mathrm d|\Lambda|(\bar h)
\\
&\le
V_\Lambda |\sigma(0)|
+
L_\sigma
\int_{B_{\overline{\mathcal U}}}
|\langle \bar u,\bar h\rangle_{\overline{\mathcal U}}|
\,\mathrm d|\Lambda|(\bar h)
\\
&\le
V_\Lambda
\Bigl(|\sigma(0)|+L_\sigma\|\bar u\|_{\overline{\mathcal U}}\Bigr)
\\
&=
V_\Lambda
\Bigl(|\sigma(0)|+L_\sigma(1+\|u\|_{\mathcal U}^2)^{1/2}\Bigr).
\end{aligned}
\]
Since \(\rho\) has finite \(q\)-th moment, \(\mathscr G^\dagger\in L^q_\rho(\mathcal V)\).

Next, define $\bar P_m^{\mathcal U}\bar u:=(1,P_m^{\mathcal U}u)$, 
and let
\(\mathscr G_m:\mathcal U\to\mathcal V\) be the input-truncated version of
\(\mathscr G^\dagger\) obtained by replacing \(\bar u\) with
\(\bar P_m^{\mathcal U}\bar u\):
\[
\mathscr G_m(u)
:=
\int_{B_{\overline{\mathcal U}}}
\sigma\bigl(\langle \bar P_m^{\mathcal U}\bar u,\bar h\rangle_{\overline{\mathcal U}}\bigr)\,
\lambda(\bar h)\,\mathrm d|\Lambda|(\bar h),
\qquad u\in\mathcal U.
\]
Then, for every \(u\in\mathcal U\), we have
\[
\begin{aligned}
\|\mathscr G^\dagger(u)-\mathscr G_m(u)\|_{\mathcal V}
&\le
\int_{B_{\overline{\mathcal U}}}
\left|
\sigma\bigl(\langle \bar u,\bar h\rangle_{\overline{\mathcal U}}\bigr)
-
\sigma\bigl(\langle \bar P_m^{\mathcal U}\bar u,\bar h\rangle_{\overline{\mathcal U}}\bigr)
\right|
\,\mathrm d|\Lambda|(\bar h)
\\
&\le
L_\sigma
\int_{B_{\overline{\mathcal U}}}
\left|
\langle \bar u-\bar P_m^{\mathcal U}\bar u,\bar h\rangle_{\overline{\mathcal U}}
\right|
\,\mathrm d|\Lambda|(\bar h)
\\
&=
L_\sigma
\int_{B_{\overline{\mathcal U}}}
\left|
\langle u-P_m^{\mathcal U}u,h\rangle_{\mathcal U}
\right|
\,\mathrm d|\Lambda|(\bar h)
\\
&\le
L_\sigma V_\Lambda \|u-P_m^{\mathcal U}u\|_{\mathcal U}.
\end{aligned}
\]
Hence
\begin{equation}\label{middle-term-Lq}
\|\mathscr G^\dagger-\mathscr G_m\|_{L^q_\rho(\mathcal V)}
\le
L_\sigma V_\Lambda
\bigl\|I-P_m^{\mathcal U}\bigr\|_{L^q_\rho(\mathcal U)}.
\end{equation}

Now set $\nu:=\frac{|\Lambda|}{V_\Lambda},$ which is a probability measure on \(B_{\overline{\mathcal U}}\).
For \(\bar h=(b,h)\in B_{\overline{\mathcal U}}\), write
\[
h_i:=\langle h,\phi_i\rangle_{\mathcal U},
\qquad i=1,\dots,m,
\]
and define
\[
\phi_{\bar h}(x)
:=
\sigma\Bigl(b+\sum_{i=1}^m x_i h_i\Bigr)\lambda(\bar h),
\qquad x\in\mathbb R^m.
\]
For each fixed \(x\in\mathbb R^m\), the map $\bar h\mapsto \sigma\Bigl(b+\sum_{i=1}^m x_i h_i\Bigr)$
is Borel measurable on \(B_{\overline{\mathcal U}}\), since \(b\) and the coordinates
\(h_i\) are weakly continuous. Since \(\lambda\) is strongly measurable, it follows that
\(\bar h\mapsto \phi_{\bar h}(x)\) is strongly measurable as a \(\mathcal V\)-valued map.
Since, for each $\bar h\in B_{\overline {\mathcal{U}}}$, the map
$x\mapsto \phi_{\bar h}(x)$ is continuous, the Carath\'eodory
measurability theorem further implies that
$(x,\bar h)\mapsto \phi_{\bar h}(x)$ is jointly strongly measurable.
Moreover,
\[
\|\phi_{\bar h}(x)\|_{\mathcal V}
=
\left|\sigma\Bigl(b+\sum_{i=1}^m x_i h_i\Bigr)\right|
\le
|\sigma(0)|+L_\sigma(1+|x|^2)^{1/2}.
\]
Therefore the Bochner integral
\[
\Phi(x):=\int_{B_{\overline{\mathcal U}}}\phi_{\bar h}(x)\,\mathrm d\nu(\bar h)
\]
is well defined for every \(x\in\mathbb R^m\).

Furthermore, for every \(x,y\in\mathbb R^m\), we have
\[
\begin{aligned}
\|\Phi(x)-\Phi(y)\|_{\mathcal V}
&\le
\int_{B_{\overline{\mathcal U}}}
\|\phi_{\bar h}(x)-\phi_{\bar h}(y)\|_{\mathcal V}\,\mathrm d\nu(\bar h)
\\
&\le
L_\sigma
\int_{B_{\overline{\mathcal U}}}
|(x-y)\cdot E_m^{\mathcal U}h|\,\mathrm d\nu(\bar h)
\\
&\le
L_\sigma |x-y|.
\end{aligned}
\]
Thus \(\Phi\) is Lipschitz, hence measurable. Also,
\[
\|\Phi(x)\|_{\mathcal V}
\le
\int_{B_{\overline{\mathcal U}}}\|\phi_{\bar h}(x)\|_{\mathcal V}\,\mathrm d\nu(\bar h)
\le
|\sigma(0)|+L_\sigma(1+|x|^2)^{1/2}.
\]
Consequently,
\[
\begin{aligned}
\|\Phi\|_{L^q(\mathbb R^m,\mu_m;\mathcal V)}
&\le
|\sigma(0)|
+
L_\sigma
\left(
\int_{\mathbb R^m}(1+|x|^2)^{q/2}\,\mathrm d\mu_m(x)
\right)^{1/q}
\\
&=
|\sigma(0)|
+
L_\sigma
\left(
\int_{\mathcal U}(1+|E_m^{\mathcal U}u|^2)^{q/2}\,\mathrm d\rho(u)
\right)^{1/q}
\\
&=
|\sigma(0)|
+
L_\sigma
\left(
\int_{\mathcal U}(1+\|P_m^{\mathcal U}u\|_{\mathcal U}^2)^{q/2}\,\mathrm d\rho(u)
\right)^{1/q}
\\
&\le
|\sigma(0)|+L_\sigma \bar M_q.
\end{aligned}
\]
Hence \(\Phi\in L^q(\mathbb R^m,\mu_m;\mathcal V)\). By construction,
\[
\mathscr G_m(u)=V_\Lambda\,\Phi(E_m^{\mathcal U}u),
\qquad u\in\mathcal U.
\]

Let \(\bar h_1,\dots,\bar h_N\) be independent random samples with law \(\nu\), and define
\[
\phi_{m,N}(x)
:=
\frac{V_\Lambda}{N}\sum_{k=1}^N \phi_{\bar h_k}(x),
\qquad x\in\mathbb R^m.
\]
Then, for every \(x\in\mathbb R^m\), we have $\mathbb E\,\phi_{m,N}(x)=V_\Lambda\,\Phi(x).$

We next estimate the sampling error. Fix \(x\in\mathbb R^m\), and set
\[
Y_k(x):=\phi_{\bar h_k}(x)-\Phi(x),
\qquad k=1,\dots,N.
\]
Then \(Y_1(x),\dots,Y_N(x)\) are independent mean-zero \(\mathcal V\)-valued random variables.
Since \(\mathcal V\) is a Hilbert space and \(q\ge2\), the Hilbert-space analogue of the Marcinkiewicz--Zygmund inequality (Proposition \ref{Marcinkiewicz--Zygmund}) yields a constant
\(A_q>0\), depending only on \(q\), such that
\[
\mathbb E\Bigl\|\sum_{k=1}^N Y_k(x)\Bigr\|_{\mathcal V}^q
\le
A_q\,
\mathbb E\Bigl(\sum_{k=1}^N \|Y_k(x)\|_{\mathcal V}^2\Bigr)^{q/2}.
\]
Using the inequality $\Bigl(\sum_{k=1}^N a_k^2\Bigr)^{q/2}
\le
N^{q/2-1}\sum_{k=1}^N |a_k|^q$,
we obtain
\[
\begin{aligned}
\mathbb E\|\phi_{m,N}(x)-V_\Lambda\Phi(x)\|_{\mathcal V}^q
&=
\frac{V_\Lambda^q}{N^q}
\mathbb E\Bigl\|\sum_{k=1}^N Y_k(x)\Bigr\|_{\mathcal V}^q
\\
&\le
\frac{A_qV_\Lambda^q}{N^q}
\mathbb E\Bigl(\sum_{k=1}^N \|Y_k(x)\|_{\mathcal V}^2\Bigr)^{q/2}
\\
&\le
\frac{A_qV_\Lambda^q}{N^{q/2+1}}
\sum_{k=1}^N \mathbb E\|Y_k(x)\|_{\mathcal V}^q
\\
&=
\frac{A_qV_\Lambda^q}{N^{q/2}}
\mathbb E\|Y_1(x)\|_{\mathcal V}^q.
\end{aligned}
\]
Moreover, 
\[
\|Y_1(x)\|_{\mathcal V}
\le
\|\phi_{\bar h_1}(x)\|_{\mathcal V}+\|\Phi(x)\|_{\mathcal V}
\le
2\bigl(|\sigma(0)|+L_\sigma(1+|x|^2)^{1/2}\bigr).
\]
Therefore
\[
\mathbb E\|\phi_{m,N}(x)-V_\Lambda\Phi(x)\|_{\mathcal V}^q
\le
\frac{C_q^qV_\Lambda^q}{N^{q/2}}
\bigl(|\sigma(0)|+L_\sigma(1+|x|^2)^{1/2}\bigr)^q,
\]
where \(C_q:=2A_q^{1/q}\) depends only on \(q\).

Integrating with respect to \(\mu_m\) and applying Fubini's theorem yields
\[
\begin{aligned}
\mathbb E
\|\phi_{m,N}-V_\Lambda\Phi\|_{L^q(\mathbb R^m,\mu_m;\mathcal V)}^q
&=
\int_{\mathbb R^m}
\mathbb E\|\phi_{m,N}(x)-V_\Lambda\Phi(x)\|_{\mathcal V}^q
\,\mathrm d\mu_m(x)
\\
&\le
\frac{C_q^qV_\Lambda^q}{N^{q/2}}
\int_{\mathbb R^m}
\bigl(|\sigma(0)|+L_\sigma(1+|x|^2)^{1/2}\bigr)^q
\,\mathrm d\mu_m(x)
\\
&\le
\frac{C_q^qV_\Lambda^q}{N^{q/2}}
\bigl(|\sigma(0)|+L_\sigma\bar M_q\bigr)^q.
\end{aligned}
\]
Therefore there exists a realization of \(\bar h_1,\dots,\bar h_N\) such that
\begin{equation}\label{sampling-term-Lq}
\|\phi_{m,N}-V_\Lambda\Phi\|_{L^q(\mathbb R^m,\mu_m;\mathcal V)}
\le
\frac{C_qV_\Lambda\bigl(|\sigma(0)|+L_\sigma\bar M_q\bigr)}{\sqrt N}.
\end{equation}

Fix such a realization. Write
\[
\bar h_k=(\alpha_k,h_k),
\quad
a_k:=E_m^{\mathcal U}h_k\in\mathbb R^m,
\quad
\beta_k:=\frac{V_\Lambda}{N}E_n^{\mathcal V}\lambda(\bar h_k)\in\mathbb R^n,
\quad k=1,\dots,N,
\]
and define
\[
g_{m,n,N}(x)
:=
\sum_{k=1}^N \beta_k\,\sigma(\alpha_k+a_k\cdot x),
\qquad x\in\mathbb R^m.
\]
Since \(\bar h_k\in B_{\overline{\mathcal U}}\),
\[
\alpha_k^2+|a_k|^2
\le
\alpha_k^2+\|h_k\|_{\mathcal U}^2
=
\|\bar h_k\|_{\overline{\mathcal U}}^2
\le 1.
\]
Also,
\[
|\beta_k|
=
\frac{V_\Lambda}{N}|E_n^{\mathcal V}\lambda(\bar h_k)|
=
\frac{V_\Lambda}{N}\|P_n^{\mathcal V}\lambda(\bar h_k)\|_{\mathcal V}
\le
\frac{V_\Lambda}{N},
\]
and therefore $\sum_{k=1}^N |\beta_k|\le V_\Lambda.$
Thus
\[
g_{m,n,N}\in \mathscr G_{m,n,N}(V_\Lambda)
=
\mathscr G_{m,n,N}\bigl(\|\mathscr G^\dagger\|_{\mathbb V(\mathcal U;\mathcal V)}\bigr).
\]

Moreover, for every \(x\in\mathbb R^m\),
\[
\begin{aligned}
|g_{m,n,N}(x)|
&\le
\sum_{k=1}^N |\beta_k|\,
\left|\sigma(\alpha_k+a_k\cdot x)\right|
\\
&\le
V_\Lambda\bigl(|\sigma(0)|+L_\sigma(1+|x|^2)^{1/2}\bigr),
\end{aligned}
\]
so \(g_{m,n,N}\in L^q(\mathbb R^m,\mu_m;\mathbb R^n)\).

By construction,
\[
D_n^{\mathcal V}g_{m,n,N}(x)
=
\frac{V_\Lambda}{N}
\sum_{k=1}^N
P_n^{\mathcal V}\lambda(\bar h_k)\,
\sigma(\alpha_k+a_k\cdot x)
=
P_n^{\mathcal V}\phi_{m,N}(x).
\]
Therefore,
\[
D_n^{\mathcal V}g_{m,n,N}(E_m^{\mathcal U}u)
=
P_n^{\mathcal V}\phi_{m,N}(E_m^{\mathcal U}u),
\qquad u\in\mathcal U.
\]

Finally, by the decomposition $\mathscr G^\dagger-P_n^{\mathcal V}\mathscr G_m
=
(I-P_n^{\mathcal V})\mathscr G^\dagger
+
P_n^{\mathcal V}(\mathscr G^\dagger-\mathscr G_m),$
 the triangle inequality, and \(\|P_n^{\mathcal V}\|=1\), we deduce that
\[
\begin{aligned}
&\quad\ \bigl\|\mathscr G^\dagger-D_n^{\mathcal V}\circ g_{m,n,N}\circ E_m^{\mathcal U}\bigr\|_{L^q_\rho(\mathcal V)}
\\
&\le
\|\mathscr G^\dagger-P_n^{\mathcal V}\mathscr G_m\|_{L^q_\rho(\mathcal V)}
+
\|P_n^{\mathcal V}\mathscr G_m
-
D_n^{\mathcal V}\circ g_{m,n,N}\circ E_m^{\mathcal U}\|_{L^q_\rho(\mathcal V)}
\\
&\le
\|(I-P_n^{\mathcal V})\mathscr G^\dagger\|_{L^q_\rho(\mathcal V)}
+
\|\mathscr G^\dagger-\mathscr G_m\|_{L^q_\rho(\mathcal V)}
+
\|\mathscr G_m-\phi_{m,N}\circ E_m^{\mathcal U}\|_{L^q_\rho(\mathcal V)}.
\end{aligned}
\]
The middle term is controlled by \eqref{middle-term-Lq}. For the last term, since
\(\mu_m=(E_m^{\mathcal U})_\#\rho\) and \(\mathscr G_m=V_\Lambda\Phi\circ E_m^{\mathcal U}\), \eqref{sampling-term-Lq} gives
\[
\begin{aligned}
\|\mathscr G_m-\phi_{m,N}\circ E_m^{\mathcal U}\|_{L^q_\rho(\mathcal V)}
&=
\|V_\Lambda\Phi-\phi_{m,N}\|_{L^q(\mathbb R^m,\mu_m;\mathcal V)}
\\
&\le
\frac{C_qV_\Lambda\bigl(|\sigma(0)|+L_\sigma\bar M_q\bigr)}{\sqrt N}.
\end{aligned}
\]
Combining these bounds yields
\[
\begin{aligned}
&\bigl\|\mathscr G^\dagger-D_n^{\mathcal V}\circ g_{m,n,N}\circ E_m^{\mathcal U}\bigr\|_{L^q_\rho(\mathcal V)}
\\
&\quad\le
\|(I-P_n^{\mathcal V})\mathscr G^\dagger\|_{L^q_\rho(\mathcal V)}
+
\left(
L_\sigma\bigl\|I-P_m^{\mathcal U}\bigr\|_{L^q_\rho(\mathcal U)}
+
\frac{C_q\bigl(|\sigma(0)|+L_\sigma\bar M_q\bigr)}{\sqrt N}
\right)\|\mathscr G^\dagger\|_{\mathbb V(\mathcal U;\mathcal V)}.
\end{aligned}
\]
This completes the proof.

\section{Proof of Theorem \ref{thm:generalization-path}} \label{Sec: proof 2}

We first construct a deterministic comparator. Since $\|P_n^{\mathcal V}\mathscr G^\dagger\|_{\mathbb V(\mathcal U;\mathcal V)}
\leq
\|\mathscr G^\dagger\|_{\mathbb V(\mathcal U;\mathcal V)},$ Theorem~\ref{thm:finite-width-approximation}, applied with \(q=2\) to
\(P_n^{\mathcal V}\mathscr G^\dagger\), yields some
\[
\Theta^\sharp\in\Theta_{m,n,N}(\|\mathscr G^\dagger\|_{\mathbb V(\mathcal U;\mathcal V)})
\]
such that, with \(\mathscr G^\sharp:=\mathscr G_{\Theta^\sharp}\),
\[
\|\mathscr G^\sharp-P_n^{\mathcal V}\mathscr G^\dagger\|_{L_\rho^2(\mathcal V)}
\leq
L_\sigma \|\mathscr G^\dagger\|_{\mathbb V(\mathcal U;\mathcal V)}
\|I-P_m^{\mathcal U}\|_{L_\rho^2(\mathcal U)}
+
\frac{C_2L_\sigma \|\mathscr G^\dagger\|_{\mathbb V(\mathcal U;\mathcal V)}\overline M_2}{\sqrt N}.
\]
Since \(\overline M_2^2\leq 2C_u^2\), it follows that
\begin{equation}\label{temp7}
\|\mathscr G^\sharp-P_n^{\mathcal V}\mathscr G^\dagger\|_{L_\rho^2(\mathcal V)}^2
\leq
2L_\sigma^2 \|\mathscr G^\dagger\|_{\mathbb V(\mathcal U;\mathcal V)}^2
\|I-P_m^{\mathcal U}\|_{L_\rho^2(\mathcal U)}^2
+
\frac{4C_2^2C_u^2L_\sigma^2\|\mathscr G^\dagger\|_{\mathbb V(\mathcal U;\mathcal V)}^2}{N}.
\end{equation}


Next, write
\[
\bar u:=(1,u)\in\overline{\mathcal U},
\qquad
\omega(u):=\|\bar u\|_{\overline{\mathcal U}}
=
(1+\|u\|_{\mathcal U}^2)^{1/2}.
\]
Let
\[
\varepsilon:=v-\mathscr G^\dagger(u),
\qquad
\xi:=P_n^{\mathcal V}\varepsilon.
\]
Then
\[
\mathbb E[\xi\mid u]=0,
\qquad
\|\xi\|_{\mathcal V}\leq\|\varepsilon\|_{\mathcal V}.
\]
For the samples, write
\[
\varepsilon_k:=v_k-\mathscr G^\dagger(u_k),
\qquad
\xi_k:=P_n^{\mathcal V}\varepsilon_k,
\qquad
k=1,\ldots,K.
\]


Set $R:=C_u\sqrt{2L_{K,\delta}}$ and $S:=C_\varepsilon\sqrt{2L_{K,\delta}}$. Define 
\[
\mathcal E_0
:=
\left\{
\max_{1\leq k\leq K}\omega(u_k)\leq R,
\quad
\max_{1\leq k\leq K}\|\xi_k\|_{\mathcal V}\leq S
\right\}.
\]
The following lemma shows that the truncation event \(\mathcal E_0\) holds
with high probability.
\begin{lemma}[Truncation event] \label{lemma 1}
Under Assumption~\ref{assump},
\[
\mathbb P(\mathcal E_0)\geq 1-\frac{\delta}{3}.
\]
\end{lemma}

\begin{proof}
Since \(\|\xi_k\|_{\mathcal V}\leq\|\varepsilon_k\|_{\mathcal V}\), Assumption~\ref{assump} and the union bound give
\[
\begin{aligned}
\mathbb P(\mathcal E_0^c)
&\leq
\sum_{k=1}^K
\left[
\mathbb P\{\omega(u_k)>R\}
+
\mathbb P\{\|\xi_k\|_{\mathcal V}>S\}
\right]
\\
&\leq
4K\exp(-2L_{K,\delta})
=
\frac{\delta^2}{36K}
\leq
\frac{\delta}{3}.
\end{aligned}
\]
This proves the claim.
\end{proof}


For \(Q>0\), define the operator class
\[
\mathbb V_n(Q)
:=
\left\{
F(u)=
\int_{B_{\overline{\mathcal U}}}
\sigma(\langle\bar u,\bar h\rangle_{\overline{\mathcal U}})
\,d\Lambda(\bar h):
\Lambda\in\mathcal M(B_{\overline{\mathcal U}};\mathcal{V}_n),\
|\Lambda|(B_{\overline{\mathcal U}})\leq Q
\right\}.
\]
Since \(\sigma\) is positively homogeneous, \(\sigma(0)=0\), and hence every
\(F\in\mathbb V_n(Q)\) satisfies
\begin{equation}\label{temp2}
\|F(u)\|_{\mathcal V}
\leq
L_\sigma Q\,\omega(u).
\end{equation}
By construction,
\[
\widehat{\mathscr G}_B\in\mathbb V_n(B),
\qquad
\mathscr G^\sharp\in\mathbb V_n(\|\mathscr G^\dagger\|_{\mathbb V(\mathcal U;\mathcal V)}),
\qquad
P_n^{\mathcal V}\mathscr G^\dagger\in\mathbb V_n(\|\mathscr G^\dagger\|_{\mathbb V(\mathcal U;\mathcal V)}).
\]
Thus, with \(Q:=B+\|\mathscr G^\dagger\|_{\mathbb V(\mathcal U;\mathcal V)}\),
\[
\widehat{\mathscr G}_B-P_n^{\mathcal V}\mathscr G^\dagger,\qquad
\mathscr G^\sharp-P_n^{\mathcal V}\mathscr G^\dagger,\qquad
\widehat{\mathscr G}_B-\mathscr G^\sharp
\in\mathbb V_n(Q).
\]


We next introduce two truncated deviation terms that will be used in the
risk decomposition. Define
\begin{equation} \label{Delta_1}
    \Delta_1(Q)
:=
\sup_{F\in\mathbb V_n(Q)}
\left|
\frac1K\sum_{k=1}^K
\|F(u_k)\|_{\mathcal V}^2
\mathbf 1_{\{\omega(u_k)\leq R\}}
-
\mathbb E\left[
\|F(u)\|_{\mathcal V}^2
\mathbf 1_{\{\omega(u)\leq R\}}
\right]
\right|,
\end{equation}
and
\begin{equation} \label{Delta_2}
\begin{aligned}
    &\quad\Delta_2(Q)
:=\\
&\sup_{F\in\mathbb V_n(Q)}
\left|
\frac1K\sum_{k=1}^K
\langle F(u_k),\xi_k\rangle_{\mathcal V}
\mathbf 1_{\{\omega(u_k)\leq R,\ \|\xi_k\|_{\mathcal V}\leq S\}}
-
\mathbb E\left[
\langle F(u),\xi\rangle_{\mathcal V}
\mathbf 1_{\{\omega(u)\leq R,\ \|\xi\|_{\mathcal V}\leq S\}}
\right]
\right|.
\end{aligned}
\end{equation}
To justify measurability, write $F=F_\Lambda$ and set
\[
\mathcal M_Q
:=
\left\{
\Lambda\in \mathcal{M}(B_{\overline{\mathcal U}};\mathcal V_n)
:\|\Lambda\|_{\mathrm{TV}}\leq Q
\right\}.
\]
Since $B_{\overline{\mathcal U}}$ is compact and metrizable and
$\mathcal V_n$ is finite-dimensional, $\mathcal M_Q$ is compact and
metrizable in the weak-* topology. For fixed data, the functionals
defining \eqref{Delta_1} and \eqref{Delta_2} are continuous in $\Lambda$,
with the expectation terms following by dominated convergence.
Hence the suprema may be taken over a fixed countable dense subset of
$\mathcal M_Q$, and therefore $\Delta_1(Q)$ and $\Delta_2(Q)$ are measurable.

The next two propositions provide high-probability bounds for these two
terms.
\begin{proposition} \label{prop: Delta_1}
    With probability at least \(1-\delta/3\), 
\begin{equation*}
    \Delta_1(Q)
\leq
15 L_\sigma^2Q^2R^2
\sqrt{\frac{\log(12/\delta)}{K}}.
\end{equation*}
\end{proposition}

\begin{proposition} \label{prop: Delta_2}
    With probability at least \(1-\delta/3\),
\begin{equation*}
    \Delta_2(Q)
\leq
16L_\sigma QRS
\sqrt{\frac{\log(12/\delta)}{K}}.
\end{equation*}
\end{proposition}
The proofs of Propositions~\ref{prop: Delta_1} and~\ref{prop: Delta_2},
which require separate estimates for the two deviation terms, are deferred
to Appendices \ref{Proof of prop: Delta_1} and \ref{Proof of prop: Delta_2}, respectively.

We now derive the risk decomposition. By the empirical minimality of
\(\widehat\Theta_B\) and
\[
\Theta^\sharp\in\Theta_{m,n,N}(\|\mathscr G^\dagger\|_{\mathbb V(\mathcal U;\mathcal V)})
\subset\Theta_{m,n,N}(B),
\]
we have
\[
\frac1K\sum_{k=1}^K
\|\widehat{\mathscr G}_B(u_k)-P_n^{\mathcal V}v_k\|_{\mathcal V}^2
\leq
\frac1K\sum_{k=1}^K
\|\mathscr G^\sharp(u_k)-P_n^{\mathcal V}v_k\|_{\mathcal V}^2.
\]
Using $P_n^{\mathcal V}v_k=P_n^{\mathcal V}\mathscr G^\dagger(u_k)+\xi_k$
and expanding the squares, we obtain
\begin{equation} \label{temp1}
    \begin{aligned}
\frac1K\sum_{k=1}^K
\|\widehat{\mathscr G}_B(u_k)-P_n^{\mathcal V}\mathscr G^\dagger(u_k)\|_{\mathcal V}^2
\leq\;&
\frac1K\sum_{k=1}^K
\|\mathscr G^\sharp(u_k)-P_n^{\mathcal V}\mathscr G^\dagger(u_k)\|_{\mathcal V}^2
\\
&+
\frac2K\sum_{k=1}^K
\langle
\widehat{\mathscr G}_B(u_k)-\mathscr G^\sharp(u_k),
\xi_k
\rangle_{\mathcal V}.
\end{aligned}
\end{equation}
On the event \(\mathcal E_0\), all empirical terms in \eqref{temp1} coincide
with their truncated counterparts. Using the definitions of
\(\Delta_1(Q)\) and \(\Delta_2(Q)\), we obtain
\begin{equation} \label{temp3}
    \begin{aligned}
        \|\widehat{\mathscr G}_B-P_n^{\mv}\mathscr G^\dagger\|_{L_\rho^2(\mathcal V)}^2\leq&
        \frac1K\sum_{k=1}^K
\|\widehat{\mathscr G}_B(u_k)-P_n^{\mathcal V}\mathscr G^\dagger(u_k)\|_{\mathcal V}^2
\\&+\Delta_1(Q)+\sup_{F\in\mathbb V_n(Q)}\mathbb E\left[
\|F(u)\|_{\mathcal V}^2
\mathbf 1_{\{\omega(u)> R\}}
\right] \\
\leq& \frac1K\sum_{k=1}^K
\|\mathscr G^\sharp(u_k)-P_n^{\mathcal V}\mathscr G^\dagger(u_k)\|_{\mathcal V}^2+\frac2K\sum_{k=1}^K
\langle
\widehat{\mathscr G}_B(u_k)-\mathscr G^\sharp(u_k),
\xi_k
\rangle_{\mathcal V}\\
&+\Delta_1(Q) +\sup_{F\in\mathbb V_n(Q)}\mathbb E\left[
\|F(u)\|_{\mathcal V}^2
\mathbf 1_{\{\omega(u)> R\}}
\right] \\
\leq& \|\mathscr G^\sharp-P_n^{\mv}\mathscr{G}^{\dagger}\|_{L_\rho^2(\mathcal V)}^2+2\Delta_1(Q)+2\Delta_2(Q)
\\
&+\sup_{F\in\mathbb V_n(Q)}\mathbb E\left[
\|F(u)\|_{\mathcal V}^2
\mathbf 1_{\{\omega(u)> R\}}
\right] \\
&+2\sup_{F\in\mathbb V_n(Q)}\left|\mathbb E\left[
\langle F(u),\xi\rangle_{\mathcal V}
\mathbf 1_{\{\omega(u)\leq R,\ \|\xi\|_{\mathcal V}\leq S\}}
\right]\right|.
    \end{aligned}
\end{equation}
The quadratic tail is controlled by the estimate
\eqref{temp2} and the sub-Gaussian tail assumption:
\begin{equation} \label{temp6}
    \begin{aligned}
        \sup_{F\in\mathbb V_n(Q)}\mathbb E\left[
\|F(u)\|_{\mathcal V}^2
\mathbf 1_{\{\omega(u)> R\}}
\right]\leq& L_\sigma^2 Q^2\mathbb E\left[
\omega(u)^2
\mathbf 1_{\{\omega(u)> R\}}
\right]
\\ \leq&L_\sigma^2 Q^2\int_{0}^\infty\mathbb{P}\l(\omega(u)>\max(\sqrt{t},R)\r)\mathrm{d}t
\\ \leq&2L_\sigma^2 Q^2\int_{0}^\infty\exp\left(-\frac{\max(t,R^2)}{C_u^2}\right)\mathrm{d}t
\\ \leq&2L_\sigma^2 Q^2\l(R^2+C_u^2\r)\exp\left(-\frac{R^2}{C_u^2}\right) 
\\ =&2L_\sigma^2 Q^2C_u^2(1+2L_{K,\delta})\l(\frac{\delta}{12K}\r)^2
 \leq \frac{L_\sigma^2 Q^2C_u^2}{24K^2}\;\delta^2 L_{K,\delta}.
    \end{aligned}
\end{equation}
Moreover, for every \(F\in\mathbb V_n(Q)\), we have 
$\mathbb E\bigl|\langle F(u),\xi\rangle_{\mathcal V}\bigr|<\infty.$
Since \(\mathbb E[\xi\mid u]=0\),
\[
\mathbb E\langle F(u),\xi\rangle_{\mathcal V}
=
\mathbb E
\left\langle
F(u),\mathbb E[\xi\mid u]
\right\rangle_{\mathcal V}
=0.
\]
Consequently,
\begin{equation} \label{temp8}
    \begin{aligned}
        &\sup_{F\in\mathbb V_n(Q)}\left|\mathbb E\left[
\langle F(u),\xi\rangle_{\mathcal V}
\mathbf 1_{\{\omega(u)\leq R,\ \|\xi\|_{\mathcal V}\leq S\}}
\right]\right|\\&\qquad=\sup_{F\in\mathbb V_n(Q)}\left|\mathbb E\left[
\langle F(u),\xi\rangle_{\mathcal V}
\mathbf 1_{\{\omega(u)>R\}\cup\{\|\xi\|_{\mathcal V}>S\}}
\right]\right|
\\&\qquad\leq L_\sigma Q\,\mathbb E\left[\omega(u)\|\xi\|_{\mv}\mathbf 1_{\{\omega(u)>R\}\cup\{\|\xi\|_{\mathcal V}>S\}}\right].
    \end{aligned}
\end{equation}
By the Cauchy--Schwarz inequality and Assumption~\ref{assump},
\begin{equation} \label{temp9}
\begin{aligned}
&\mathbb E\left[
\omega(u)\|\xi\|_{\mathcal V}
\mathbf 1_{\{\omega(u)>R\}\cup\{\|\xi\|_{\mathcal V}>S\}}
\right]
\\
&\qquad\leq
\left(
\mathbb E[
\omega(u)^2\mathbf 1_{\{\omega(u)>R\}}
]
\right)^{1/2}
\left(
\mathbb E\|\xi\|_{\mathcal V}^2
\right)^{1/2}
\\
&\qquad\quad+
\left(
\mathbb E\omega(u)^2
\right)^{1/2}
\left(
\mathbb E[
\|\xi\|_{\mathcal V}^2\mathbf 1_{\{\|\xi\|_{\mathcal V}>S\}}
]
\right)^{1/2}
\\
&\qquad\leq 2\sqrt{R^2+C_u^2}\;C_\varepsilon\exp\left(-\frac{R^2}{2C_u^2}\right) + 2\sqrt{S^2+C_\varepsilon^2}\;C_u\exp\left(-\frac{S^2}{2C_\varepsilon^2}\right)
\\
&\qquad=C_uC_\varepsilon\sqrt{1+2L_{K,\delta}}\;\frac{\delta}{3K}
\leq C_uC_\varepsilon\sqrt{L_{K,\delta}}\;\frac{\delta}{\sqrt{3}K}.
\end{aligned}
\end{equation}


Finally, on the intersection of \(\mathcal E_0\) and the events in
Propositions~\ref{prop: Delta_1} and~\ref{prop: Delta_2}, which has probability
at least \(1-\delta\), substituting
\eqref{temp7}, \eqref{temp6}, \eqref{temp8}, and \eqref{temp9} into
\eqref{temp3} yields the desired bound for $\|\widehat{\mathscr G}_B-P_n^{\mathcal V}\mathscr G^\dagger
\|_{L_\rho^2(\mathcal V)}^2.$ Since both \(\widehat{\mathscr G}_B\) and
\(P_n^{\mathcal V}\mathscr G^\dagger\) take values in \(V_n\), while
\((I-P_n^{\mathcal V})\mathscr G^\dagger\) takes values in \(V_n^\perp\),
the Pythagorean identity gives
\[
\|\widehat{\mathscr G}_B-\mathscr G^\dagger\|_{L_\rho^2(\mathcal V)}^2
=
\|\widehat{\mathscr G}_B-P_n^{\mathcal V}\mathscr G^\dagger\|_{L_\rho^2(\mathcal V)}^2
+
\|(I-P_n^{\mathcal V})\mathscr G^\dagger\|_{L_\rho^2(\mathcal V)}^2.
\]
Combining the preceding estimates proves the theorem.

\appendix

\section{Auxiliary Probabilistic Results} 
This appendix collects the auxiliary pro\-b\-a\-bil\-is\-tic results used in the proofs
of the main theorems. We first prove a Mar\-cin\-kie\-wicz--Zyg\-mund inequality needed for the finite-width approximation
argument. We then recall a vector-contraction inequality and McDiarmid's
bounded differences inequality used in the generalization analysis.

\noindent\textbf{Marcinkiewicz--Zygmund inequality.}
We use the following Hilbert-space-valued Marcinkiewicz--Zygmund-type inequality in the proof of the approximation theorem. Such estimates may be standard in the probability theory of Banach spaces, but we have not found a reference that states the precise form required here. For completeness, we include
a short proof based on symmetrization and the Kahane--Khintchine inequality.
The same argument extends to type-2 Banach spaces\footnote{A Banach space \(\mathcal{X}\) is called a type-2 Banach space if there exists a constant $C_{2,\mathcal{X}}>0$ such that, for any $n\geq1$ and any $x_1,\cdots,x_n\in \mathcal{X}$, one has 
\[
\left(\mathbb{E}\left[\left\|\sum_{i=1}^n\epsilon_ix_i\right\|_{\mathcal{X}}^2\right]\right)^{1/2}\leq
C_{2,\mathcal{X}}\left(\sum_{i=1}^n\big\|x_i\big\|_{\mathcal{X}}^2\right)^{1/2},
\]
where the expectation is taken over all independent Rademacher variables $\epsilon_1,\cdots,\epsilon_n$, i.e., $\mathbb{P}(\epsilon_i=-1)=\mathbb{P}(\epsilon_i=1)=1/2$.}, with the constant depending on the type-2 constant of the underlying space. 

\begin{proposition}[Marcinkiewicz--Zygmund inequality] \label{Marcinkiewicz--Zygmund}
Let \(2\le p<\infty\), and let \(Y_1,\dots,Y_N\) be independent mean-zero random variables with values in a real Hilbert space \(\mathcal V\). Assume
that \(\mathbb E\|Y_k\|_{\mathcal V}^p<\infty\) for \(k=1,\dots,N\). Then there exists a constant \(A_p>0\), depending only on \(p\), such that
\[
\mathbb E\Bigl\|\sum_{k=1}^N Y_k\Bigr\|_{\mathcal V}^p
\le
A_p\,
\mathbb E\Bigl(\sum_{k=1}^N \|Y_k\|_{\mathcal V}^2\Bigr)^{p/2}.
\]
\end{proposition}

\begin{proof}
Let \(Y_1',\dots,Y_N'\) be an independent copy of \(Y_1,\dots,Y_N\), and let
\(\varepsilon_1,\dots,\varepsilon_N\) be independent Rademacher random variables,
independent of both \((Y_k)\) and \((Y_k')\).
We denote by \(\mathbb E'\) the expectation with respect to \((Y_k')\), and by
\(\mathbb E_\varepsilon\) the expectation with respect to \((\varepsilon_k)\).

Since \(\mathbb E'Y_k'=0\), we have
\[
\sum_{k=1}^N Y_k
=
\sum_{k=1}^N \bigl(Y_k-\mathbb E'Y_k'\bigr)
=
\mathbb E'\sum_{k=1}^N (Y_k-Y_k').
\]
Hence, by Jensen's inequality,
\[
\mathbb E\Bigl\|\sum_{k=1}^N Y_k\Bigr\|_{\mathcal V}^p
\le
\mathbb E\mathbb E'
\Bigl\|\sum_{k=1}^N (Y_k-Y_k')\Bigr\|_{\mathcal V}^p.
\]

Set
\[
Z_k:=Y_k-Y_k',
\qquad k=1,\dots,N.
\]
Then \(Z_1,\dots,Z_N\) are independent symmetric \(\mathcal V\)-valued random variables. Therefore,
\[
(Z_1,\dots,Z_N)\stackrel{d}{=}(\varepsilon_1 Z_1,\dots,\varepsilon_N Z_N),
\]
and thus
\[
\mathbb E\mathbb E'
\Bigl\|\sum_{k=1}^N (Y_k-Y_k')\Bigr\|_{\mathcal V}^p
=
\mathbb E\mathbb E'\mathbb E_\varepsilon
\Bigl\|\sum_{k=1}^N \varepsilon_k (Y_k-Y_k')\Bigr\|_{\mathcal V}^p.
\]

Applying Minkowski's inequality, we obtain
\[
\begin{aligned}
\Bigl(
\mathbb E\mathbb E'\mathbb E_\varepsilon
\Bigl\|\sum_{k=1}^N \varepsilon_k (Y_k-Y_k')\Bigr\|_{\mathcal V}^p
\Bigr)^{1/p}
&\le
\Bigl(
\mathbb E\mathbb E_\varepsilon
\Bigl\|\sum_{k=1}^N \varepsilon_k Y_k\Bigr\|_{\mathcal V}^p
\Bigr)^{1/p}
\\
&\quad+
\Bigl(
\mathbb E'\mathbb E_\varepsilon
\Bigl\|\sum_{k=1}^N \varepsilon_k Y_k'\Bigr\|_{\mathcal V}^p
\Bigr)^{1/p}.
\end{aligned}
\]
Since \((Y_k')\) has the same law as \((Y_k)\), the two terms on the right-hand side are equal. Hence
\[
\Bigl(
\mathbb E\Bigl\|\sum_{k=1}^N Y_k\Bigr\|_{\mathcal V}^p
\Bigr)^{1/p}
\le
2
\Bigl(
\mathbb E\mathbb E_\varepsilon
\Bigl\|\sum_{k=1}^N \varepsilon_k Y_k\Bigr\|_{\mathcal V}^p
\Bigr)^{1/p}.
\]

Next, conditioning on \(Y_1,\dots,Y_N\) and applying the Kahane--Khintchine inequality (in the Banach-space setting; see \cite[Theorem 6.2.5]{albiac2006topics}
and \cite[Theorem 3.11]{van2008stochastic}) with exponents \((p,2)\), we get
\[
\mathbb E_\varepsilon
\Bigl\|\sum_{k=1}^N \varepsilon_k Y_k\Bigr\|_{\mathcal V}^p
\le
K_{p,2}^p
\Bigl(
\mathbb E_\varepsilon
\Bigl\|\sum_{k=1}^N \varepsilon_k Y_k\Bigr\|_{\mathcal V}^2
\Bigr)^{p/2},
\]
where \(K_{p,2}>0\) depends only on \(p\).

Since \(\mathcal V\) is a Hilbert space,
\[
\mathbb E_\varepsilon
\Bigl\|\sum_{k=1}^N \varepsilon_k Y_k\Bigr\|_{\mathcal V}^2
=
\sum_{k=1}^N \|Y_k\|_{\mathcal V}^2.
\]
Therefore,
\[
\mathbb E_\varepsilon
\Bigl\|\sum_{k=1}^N \varepsilon_k Y_k\Bigr\|_{\mathcal V}^p
\le
K_{p,2}^p
\Bigl(\sum_{k=1}^N \|Y_k\|_{\mathcal V}^2\Bigr)^{p/2}.
\]
Taking expectation with respect to \((Y_k)\), we obtain
\[
\mathbb E\mathbb E_\varepsilon
\Bigl\|\sum_{k=1}^N \varepsilon_k Y_k\Bigr\|_{\mathcal V}^p
\le
K_{p,2}^p
\mathbb E\Bigl(\sum_{k=1}^N \|Y_k\|_{\mathcal V}^2\Bigr)^{p/2}.
\]

Combining the above estimates yields
\[
\mathbb E\Bigl\|\sum_{k=1}^N Y_k\Bigr\|_{\mathcal V}^p
\le
(2K_{p,2})^p
\mathbb E\Bigl(\sum_{k=1}^N \|Y_k\|_{\mathcal V}^2\Bigr)^{p/2}.
\]
Thus the conclusion holds with \(A_p:=(2K_{p,2})^p\).
\end{proof}

\noindent\textbf{Vector-contraction inequality.}
The following lemma is a consequence of the vector-contraction inequality
of Maurer~\cite[Theorem~2]{maurer2016vector}, specialized to standard
Gaussian comparison variables.

\begin{lemma}[Vector-contraction inequality]
\label{lem:vector-contraction}
Let \(\mathcal T\subset(\mathbb R^d)^K\) be countable and contain the zero
element, and set
\[
\mathcal D_{\mathcal T}
:=
\left\{
t_k:
t=(t_1,\ldots,t_K)\in\mathcal T,\ 
1\leq k\leq K
\right\}
\subset\mathbb R^d.
\]
Let \(\phi:\mathcal D_{\mathcal T}\to\mathbb R\) be \(L\)-Lipschitz with
\(\phi(0)=0\). Let \(r_1,\ldots,r_K\) be independent Rademacher variables,
and let \(\gamma_1,\ldots,\gamma_K\sim\mathcal N(0,I_d)\) be independent
standard Gaussian vectors in \(\mathbb R^d\). Then
\[
\mathbb E_r
\sup_{t\in\mathcal T}
\left|
\sum_{k=1}^K r_k\phi(t_k)
\right|
\leq
\sqrt{2\pi}\,L\,
\mathbb E_\gamma
\sup_{t\in\mathcal T}
\sum_{k=1}^K
\langle\gamma_k,t_k\rangle_{\mathbb R^d}.
\]
\end{lemma}

\begin{proof}
For \(t=(t_1,\ldots,t_K)\in\mathcal T\), set
\[
\psi_k(t):=\phi(t_k),
\qquad
\Phi_k(t):=Lt_k,
\qquad
k=1,\ldots,K.
\]
Since \(\phi\) is \(L\)-Lipschitz on \(\mathcal D_{\mathcal T}\),
\[
\psi_k(t)-\psi_k(t')
\leq
L|t_k-t_k'|
=
|\Phi_k(t)-\Phi_k(t')|.
\]
Hence Theorem~2 of \cite{maurer2016vector}, applied with standard Gaussian
variables, yields
\[
\mathbb E_r
\sup_{t\in\mathcal T}
\sum_{k=1}^K r_k\phi(t_k)
\leq
\sqrt{\frac{\pi}{2}}\,L\,
\mathbb E_\gamma
\sup_{t\in\mathcal T}
\sum_{k=1}^K
\langle\gamma_k,t_k\rangle_{\mathbb R^d}.
\]
Since \(0\in\mathcal T\) and \(\phi(0)=0\), it follows that $\sup_{t\in\mathcal T}
\sum_{k=1}^K r_k\phi(t_k)\geq0.$
Using the symmetry of the Rademacher variables,
\[
\begin{aligned}
\mathbb E_r
\sup_{t\in\mathcal T}
\left|
\sum_{k=1}^K r_k\phi(t_k)
\right|
&\leq
\mathbb E_r
\sup_{t\in\mathcal T}
\sum_{k=1}^K r_k\phi(t_k)
+
\mathbb E_r
\sup_{t\in\mathcal T}
\left(
-\sum_{k=1}^K r_k\phi(t_k)
\right)
\\
&=
2\mathbb E_r
\sup_{t\in\mathcal T}
\sum_{k=1}^K r_k\phi(t_k),
\end{aligned}
\]
and the conclusion follows.
\end{proof}

\noindent\textbf{McDiarmid's bounded differences inequality.} We recall the bounded differences inequality of McDiarmid
\cite[Lemma~1.2]{mcdiarmid1989method}.
\begin{lemma}[Bounded differences inequality]
\label{lem:bounded-differences}
Let \(X_1,\ldots,X_K\) be independent random variables with
\(X_k\) taking values in a set \(\mathcal X_k\), and let
\[
\Phi:\mathcal X_1\times\cdots\times\mathcal X_K\to\mathbb R
\]
be measurable and suppose that, for every \(k\),
\[
\left|
\Phi(x_1,\ldots,x_k,\ldots,x_K)
-
\Phi(x_1,\ldots,x_k',\ldots,x_K)
\right|
\leq c_k
\]
whenever the two arguments differ only in their \(k\)-th coordinate. Then, for every \(t>0\),
\[
\mathbb P\left\{
\big|\Phi(X_1,\ldots,X_K)-\mathbb E\Phi(X_1,\ldots,X_K)\big|\geq t
\right\}
\leq
2\exp\left(
-\frac{2t^2}{\sum_{k=1}^Kc_k^2}
\right).
\]
\end{lemma}

\section{Proof of Proposition \ref{prop: property}} \label{Proof of prop: property}

\noindent\textbf{Proofs of (1) and (2).}
Define the linear map $\mathcal R: \mathcal M(B_{\overline{\mathcal U}};\mathcal V) \to \mathcal V^{\mathcal U}$ by \[ (\mathcal R\Lambda)(u) := \int_{B_{\overline{\mathcal U}}} \sigma\bigl( \langle \bar u,\bar h\rangle_{\overline{\mathcal U}} \bigr)\, \mathrm d\Lambda(\bar h), \qquad u\in\mathcal U. \] This map is well defined.
Indeed, for every fixed $u\in\mathcal U$, the function $\bar h\longmapsto \sigma\bigl( \langle \bar u,\bar h\rangle_{\overline{\mathcal U}} \bigr)$  is Borel measurable on $B_{\overline{\mathcal U}}$. Moreover, 
\begin{equation} \|(\mathcal R\Lambda)(u)\|_{\mathcal V} \le \sup_{|t|\le \|\bar u\|_{\overline{\mathcal U}}} |\sigma(t)|\,|\Lambda|(B_{\overline{\mathcal U}}). \label{eq:R-pointwise-bound} \end{equation} 
By definition, $\mathbb V(\mathcal U;\mathcal V) = \operatorname{Ran}(\mathcal R).$

We next show that $\ker(\mathcal R)$ is closed in $\mathcal M(B_{\overline{\mathcal U}};\mathcal V)$ with respect to the total variation norm. Let $\Lambda_k\in\ker(\mathcal R)$ and suppose that
$|\Lambda_k-\Lambda|(B_{\overline{\mathcal U}}) \to 0$ for some $\Lambda\in\mathcal M(B_{\overline{\mathcal U}};\mathcal V)$. For every fixed $u\in\mathcal U$, since $\mathcal R\Lambda_k=0$, \eqref{eq:R-pointwise-bound} gives 
\[ \begin{aligned} \|(\mathcal R\Lambda)(u)\|_{\mathcal V} &= \|(\mathcal R(\Lambda-\Lambda_k))(u)\|_{\mathcal V} \\ &\le \sup_{|t|\le \|\bar u\|_{\overline{\mathcal U}}}|\sigma(t)|\,|\Lambda-\Lambda_k|(B_{\overline{\mathcal U}}) \longrightarrow 0. \end{aligned} \] 
Thus $(\mathcal R\Lambda)(u)=0$ for every $u\in\mathcal U$, and hence $\Lambda\in\ker(\mathcal R)$. Therefore $\ker(\mathcal R)$ is closed. Since $\mathcal M(B_{\overline{\mathcal U}};\mathcal V)$ is a Banach space, the quotient space $\mathcal M(B_{\overline{\mathcal U}};\mathcal V) /\ker(\mathcal R)$ is a Banach space with quotient norm \[ \|[\Lambda]\|_{\mathrm{quot}} := \inf_{\Gamma\in\ker(\mathcal R)} |\Lambda+\Gamma|(B_{\overline{\mathcal U}}). \] The map \[ \widetilde{\mathcal R}: \mathcal M(B_{\overline{\mathcal U}};\mathcal V) /\ker(\mathcal R) \longrightarrow \mathbb V(\mathcal U;\mathcal V), \qquad \widetilde{\mathcal R}[\Lambda] := \mathcal R\Lambda, \] is a well-defined isometric isomorphism. In particular, $\|\cdot\|_{\mathbb V(\mathcal U;\mathcal V)}$ is a norm, and $\mathbb V(\mathcal U;\mathcal V)$ is a Banach space. 

\noindent\textbf{Proof of (3).}
Since $B_{\overline{\mathcal U}}$ is compact in the weak topology and
$\mathcal V$ is a Hilbert space, the vector-valued Riesz representation
theorem gives the isometric identification \cite{1977Vector}
\begin{equation} \label{temp21}
    C(B_{\overline{\mathcal U}};\mathcal V)^*
\simeq
\mathcal M(B_{\overline{\mathcal U}};\mathcal V).
\end{equation}
Let
\[
\mathcal A_{\mathscr G}
:=
\left\{
\Lambda\in
\mathcal M(B_{\overline{\mathcal U}};\mathcal V):
\Lambda \text{ represents }\mathscr G
\right\}.
\]
By the definition of $\mathbb V(\mathcal U;\mathcal V)$,
$\mathcal A_{\mathscr G}$ is nonempty. 

We first show that
$\mathcal A_{\mathscr G}$ is weak-* closed in
$C(B_{\overline{\mathcal U}};\mathcal V)^*$.
For $u\in\mathcal U$ and $v\in\mathcal V$, define
\[
F_{u,v}(\bar h)
:=
\sigma\bigl(
\langle\bar u,\bar h\rangle_{\overline{\mathcal U}}
\bigr)v,
\qquad
\bar h\in B_{\overline{\mathcal U}}.
\]
Since $\bar h\mapsto
\langle\bar u,\bar h\rangle_{\overline{\mathcal U}}$ is continuous with respect to the weak topology on
$B_{\overline{\mathcal U}}$ and $\sigma$ is continuous, we have $F_{u,v}\in C(B_{\overline{\mathcal U}};\mathcal V).$ Moreover, by \eqref{temp21},
\[
\langle\Lambda,F_{u,v}\rangle
=
\left\langle
\int_{B_{\overline{\mathcal U}}}
\sigma\bigl(
\langle\bar u,\bar h\rangle_{\overline{\mathcal U}}
\bigr)
\,\mathrm d\Lambda(\bar h),
v
\right\rangle_{\mathcal V}.
\]
Consequently, $\Lambda\in\mathcal A_{\mathscr G}$ if and only if
\[
\langle\Lambda,F_{u,v}\rangle
=
\langle\mathscr G(u),v\rangle_{\mathcal V},
\qquad
\text{for all }\; u\in\mathcal U \text{ and }  v\in\mathcal V.
\]
Hence
\[
\mathcal A_{\mathscr G}
=
\bigcap_{u\in\mathcal U,\;v\in\mathcal V}
\left\{
\Lambda\in C(B_{\overline{\mathcal U}};\mathcal V)^*:
\langle\Lambda,F_{u,v}\rangle
=
\langle\mathscr G(u),v\rangle_{\mathcal V}
\right\}.
\]
For each $u\in\mathcal U$ and $v\in\mathcal V$, the map $\Lambda\mapsto\langle\Lambda,F_{u,v}\rangle$
is weak-* continuous. Therefore,
\[
\left\{
\Lambda\in C(B_{\overline{\mathcal U}};\mathcal V)^*:
\langle\Lambda,F_{u,v}\rangle
=
\langle\mathscr G(u),v\rangle_{\mathcal V}
\right\}
\]
is weak-* closed. Thus $\mathcal A_{\mathscr G}$ is weak-* closed.

Consider
\[
\mathcal C_{\mathscr G}
:=
\mathcal A_{\mathscr G}
\cap
\left\{
\Lambda\in C(B_{\overline{\mathcal U}};\mathcal V)^*:
\|\Lambda\|_{C(B_{\overline{\mathcal U}};\mathcal V)^*}
\le \|\mathscr G\|_{\mathbb V(\mathcal U;\mathcal V)}+1
\right\}.
\]
By the Banach--Alaoglu theorem, the closed ball in
$C(B_{\overline{\mathcal U}};\mathcal V)^*$ is weak-* compact.
Since $\mathcal A_{\mathscr G}$ is weak-* closed,
$\mathcal C_{\mathscr G}$ is weak-* compact.

The dual norm on
$C(B_{\overline{\mathcal U}};\mathcal V)^*$ is weak-* lower
semicontinuous, since
\[
\|\Lambda\|_{C(B_{\overline{\mathcal U}};\mathcal V)^*}
=
\sup_{\substack{
F\in C(B_{\overline{\mathcal U}};\mathcal V)\\
\|F\|_\infty\le1}}
|\langle\Lambda,F\rangle|
\]
is the supremum of weak-* continuous functions of $\Lambda$.
Consequently, since $\mathcal C_{\mathscr G}$ is nonempty and weak-*
compact, there exists
$\Lambda_{\mathscr G}\in\mathcal C_{\mathscr G}$ such that
\[
|\Lambda_{\mathscr G}|(B_{\overline{\mathcal U}})
=
\min_{\Lambda\in\mathcal C_{\mathscr G}}
|\Lambda|(B_{\overline{\mathcal U}}).
\]
This proves the assertion.

\section{Proof of Proposition \ref{prop: characterization}} \label{Proof of prop: characterization}

Set $a_\sigma(u)
:=
\sup_{|t|\le \|\bar u\|_{\overline{\mathcal U}}}
|\sigma(t)|.$ By assumption, \(a_\sigma\in L^p_\gamma(\mathbb R)\).

We first note that
\(\mathbb V(\mathcal U;\mathcal V)\subset L^p_\gamma(\mathcal V)\).
Let
\(\Lambda\in\mathcal M(B_{\overline{\mathcal U}};\mathcal V)\)
be a representing measure of \(\mathscr G\). Then
\[
\|\mathscr G(u)\|_{\mathcal V}
\le
a_\sigma(u)\,
|\Lambda|(B_{\overline{\mathcal U}}),
\qquad u\in\mathcal U.
\]
The continuity of \(\sigma\) also implies that \(\mathscr G:\mathcal U\to\mathcal V\) is continuous, hence strongly measurable. Therefore,
\[
\|\mathscr G\|_{L^p_\gamma(\mathcal V)}
\le
\|a_\sigma\|_{L^p_\gamma}
|\Lambda|(B_{\overline{\mathcal U}}).
\]
Taking the infimum over all representing measures yields
\[
\|\mathscr G\|_{L^p_\gamma(\mathcal V)}
\le
\|a_\sigma\|_{L^p_\gamma}
\|\mathscr G\|_{\mathbb V(\mathcal U;\mathcal V)}.
\]

We next prove that
\[
\left\{
\mathscr G\in\mathbb V(\mathcal U;\mathcal V):
\|\mathscr G\|_{\mathbb V(\mathcal U;\mathcal V)}
\le1
\right\}
\subset
\overline{\operatorname{conv}(\mathcal D_\sigma)}
^{\,L^p_\gamma(\mathcal V)}.
\]
Let
\(\mathscr G\in\mathbb V(\mathcal U;\mathcal V)\)
with
\(\|\mathscr G\|_{\mathbb V(\mathcal U;\mathcal V)}\le1\).
By Proposition~\ref{prop: property}(3), there exists a representing
measure
\(\Lambda\in
\mathcal M(B_{\overline{\mathcal U}};\mathcal V)\)
such that $\|\Lambda\|_{\mathrm{TV}}
=
\|\mathscr G\|_{\mathbb V(\mathcal U;\mathcal V)}
\le1.$ By the polar decomposition of vector measures,
\[
\mathrm d\Lambda(\bar h)
=
\lambda(\bar h)\,\mathrm d|\Lambda|(\bar h),
\qquad
\|\lambda(\bar h)\|_{\mathcal V}=1
\quad
|\Lambda|\text{-a.e.}
\]

For \(\bar h\in B_{\overline{\mathcal U}}\), define $\phi_{\bar h}(u)
:=
\sigma\bigl(
\langle\bar u,\bar h\rangle_{\overline{\mathcal U}}
\bigr).$ The map
\[
B_{\overline{\mathcal U}}
\ni\bar h
\longmapsto
\phi_{\bar h}
\in L^p_\gamma(\mathbb R)
\]
is continuous when
\(B_{\overline{\mathcal U}}\) is endowed with the weak topology.
Indeed, if
\(\bar h_j\rightharpoonup\bar h\) weakly, then
\(\phi_{\bar h_j}(u)\to\phi_{\bar h}(u)\) for every \(u\), while
\[
|\phi_{\bar h_j}(u)-\phi_{\bar h}(u)|
\le2a_\sigma(u).
\]
Hence the claim follows from the dominated convergence theorem.

Since \(\lambda\) is strongly measurable, the map
$\Phi(\bar h)(u)
:=
\phi_{\bar h}(u)\lambda(\bar h)$ is strongly measurable as an \(L^p_\gamma(\mathcal V)\)-valued map, and
\[
\|\Phi(\bar h)\|_{L^p_\gamma(\mathcal V)}
\le
\|a_\sigma\|_{L^p_\gamma}
\qquad
|\Lambda|\text{-a.e.}
\]
Hence \(\Phi\) is Bochner integrable with respect to \(|\Lambda|\), and
\[
\mathscr G
=
\int_{B_{\overline{\mathcal U}}}
\Phi(\bar h)\,
\mathrm d|\Lambda|(\bar h)
\qquad
\text{in }L^p_\gamma(\mathcal V).
\]

Since
\(\Phi(\bar h)\in\mathcal D_\sigma\)
for \(|\Lambda|\)-almost every \(\bar h\), if
\(\|\Lambda\|_{\mathrm{TV}}>0\), then
\[
\frac{\mathscr G}{\|\Lambda\|_{\mathrm{TV}}}
=
\int_{B_{\overline{\mathcal U}}}
\Phi(\bar h)\,
\frac{\mathrm d|\Lambda|(\bar h)}
{\|\Lambda\|_{\mathrm{TV}}}
\in
\overline{\operatorname{conv}(\mathcal D_\sigma)}
^{\,L^p_\gamma(\mathcal V)}.
\]
If \(\|\Lambda\|_{\mathrm{TV}}=0\), then \(\mathscr G=0\).
Since \(0\in\operatorname{conv}(\mathcal D_\sigma)\) and
\(\|\Lambda\|_{\mathrm{TV}}\le1\), we conclude that $\mathscr G
\in
\overline{\operatorname{conv}(\mathcal D_\sigma)}
^{\,L^p_\gamma(\mathcal V)}.$

We now prove the reverse inclusion. Let $\mathscr G\in
\overline{\operatorname{conv}(\mathcal D_\sigma)}
^{\,L^p_\gamma(\mathcal V)}.$ 
Choose
\[
\mathscr G_j\in\operatorname{conv}(\mathcal D_\sigma)
\qquad\text{such that}\qquad
\mathscr G_j\to\mathscr G
\quad\text{in }L^p_\gamma(\mathcal V).
\]
Each \(\mathscr G_j\) can be written as
\[
\mathscr G_j(u)
=
\sum_{\ell=1}^{N_j}
c_{j,\ell}\,
\sigma\bigl(
\langle\bar u,\bar h_{j,\ell}\rangle_{\overline{\mathcal U}}
\bigr)v_{j,\ell},
\]
where
\[
c_{j,\ell}\ge0,\qquad
\sum_{\ell=1}^{N_j}c_{j,\ell}=1,\qquad
\bar h_{j,\ell}\in B_{\overline{\mathcal U}},
\qquad
v_{j,\ell}\in S_{\mathcal V}.
\]
Define
\[
\Lambda_j
:=
\sum_{\ell=1}^{N_j}
c_{j,\ell}v_{j,\ell}\delta_{\bar h_{j,\ell}}
\in
\mathcal M(B_{\overline{\mathcal U}};\mathcal V).
\]
Then \(\Lambda_j\) represents \(\mathscr G_j\) and
\[
\|\Lambda_j\|_{\mathrm{TV}}
\le
\sum_{\ell=1}^{N_j}
c_{j,\ell}\|v_{j,\ell}\|_{\mathcal V}
=1.
\]

The vector-valued Riesz
representation theorem gives
\[
C(B_{\overline{\mathcal U}};\mathcal V)^*
\simeq
\mathcal M(B_{\overline{\mathcal U}};\mathcal V).
\]
Since \(B_{\overline{\mathcal U}}\) is compact and metrizable in the
weak topology and \(\mathcal V\) is separable,
\(C(B_{\overline{\mathcal U}};\mathcal V)\) is separable.
By the Banach--Alaoglu theorem, the closed unit ball of \(C(B_{\overline{\mathcal U}};\mathcal V)^*\) is weak-* compact. Since the predual is separable, this topology is metrizable on the unit ball.
Hence, after passing
to a subsequence,
there exists
\(\Lambda\in\mathcal M(B_{\overline{\mathcal U}};\mathcal V)\)
with $\|\Lambda\|_{\mathrm{TV}}\le1$ 
such that
\[
\Lambda_j \to \Lambda
\qquad\text{in the weak-* topology of }
C(B_{\overline{\mathcal U}};\mathcal V)^*.
\]

Define
\[
\mathscr G_\Lambda(u)
:=
\int_{B_{\overline{\mathcal U}}}
\sigma\bigl(
\langle\bar u,\bar h\rangle_{\overline{\mathcal U}}
\bigr)
\,\mathrm d\Lambda(\bar h).
\]
For every \(u\in\mathcal U\) and \(w\in\mathcal V\), the map $\bar h\mapsto
\sigma\bigl(
\langle\bar u,\bar h\rangle_{\overline{\mathcal U}}
\bigr)w$ belongs to \(C(B_{\overline{\mathcal U}};\mathcal V)\). Therefore,
\[
\langle\mathscr G_j(u),w\rangle_{\mathcal V}
\longrightarrow
\langle\mathscr G_\Lambda(u),w\rangle_{\mathcal V}
\qquad
\text{for every }u\in\mathcal U,\ w\in\mathcal V.
\]
Thus, for every \(u\in\mathcal U\), \(\mathscr G_j(u)\) converges weakly to \(\mathscr G_\Lambda(u)\) in \(\mathcal V\).

On the other hand, since
\(\mathscr G_j\to\mathscr G\) in \(L^p_\gamma(\mathcal V)\) and
\(1\le p<\infty\), after passing to a further subsequence,
\[
\mathscr G_j(u)\to\mathscr G(u)
\qquad
\text{in }\mathcal V
\quad\text{for }\gamma\text{-a.e. }u.
\]
Hence, by uniqueness of the weak limit, $\mathscr G(u)=\mathscr G_\Lambda(u)$ for $\gamma$-a.e. $u$. 
Thus $\mathscr G=\mathscr G_\Lambda$ in $L^p_\gamma(\mathcal V)$ and $\|\mathscr G_\Lambda\|_{\mathbb V(\mathcal U;\mathcal V)}
\le
\|\Lambda\|_{\mathrm{TV}}
\le1.$

This completes the proof.

\section{Proof of Proposition \ref{prop: linear characterization}}
\label{Proof of prop: linear characterization}
Suppose first that $T\in\mathcal S_1(\mathcal U,\mathcal V)$. By the
singular value decomposition,
\[
Tu=\sum_{j\ge1}s_j
\langle u,e_j\rangle_{\mathcal U}f_j,
\qquad
\sum_{j\ge1}s_j
=
\|T\|_{\mathcal S_1(\mathcal U,\mathcal V)},
\]
where $(s_j)_{j\ge1}$ are the singular values of $T$ and
$\{e_j\}\subset\mathcal U$, $\{f_j\}\subset\mathcal V$ are orthonormal
systems. Since $t=t_+-(-t)_+$,
\[
Tu
=
\sum_{j\ge1}s_j
\left[
\operatorname{ReLU}(\langle u,e_j\rangle_{\mathcal U})
-
\operatorname{ReLU}(-\langle u,e_j\rangle_{\mathcal U})
\right]f_j.
\]
Thus $T$ admits a representing measure with total variation
$2\sum_j s_j$, and hence
\[
\|T\|_{\mathbb V(\mathcal U;\mathcal V)}
\le
2\|T\|_{\mathcal S_1(\mathcal U,\mathcal V)}.
\]

Conversely, let $T\in\mathbb V(\mathcal U;\mathcal V)$, and let
$\Lambda\in\mathcal M(B_{\overline{\mathcal U}};\mathcal V)$ be any
representing measure. Writing $\bar h=(b,h)$, for every $u\in\mathcal U$,
the linearity of $T$ and dominated convergence yield
\begin{equation*}
    \begin{aligned}
        Tu
&=
\lim_{r\to\infty}\frac{T(ru)}{r}
\\&=
\lim_{r\to\infty}\int_{B_{\overline{\mathcal U}}}
\left(
\frac{b}{r}+\langle u,h\rangle_{\mathcal U}
\right)_+
\,\mathrm d\Lambda(\bar h)
\\&=
\int_{B_{\overline{\mathcal U}}}
\l(\langle u,h\rangle_{\mathcal U}\r)_{+}\,
\mathrm d\Lambda(\bar h),
    \end{aligned}
\end{equation*}
Similarly,
\[
\begin{aligned}
-Tu
&=
\lim_{r\to\infty}\frac{T(-ru)}{r}
=
\int_{B_{\overline{\mathcal U}}}
(-\langle u,h\rangle_{\mathcal U})_+\,
\mathrm d\Lambda(\bar h).
\end{aligned}
\]
Since $t_+-(-t)_+=t$, it follows that
\[
2Tu
=
\int_{B_{\overline{\mathcal U}}}
\langle u,h\rangle_{\mathcal U}\,
\mathrm d\Lambda(\bar h).
\]
Let $\mathrm d\Lambda=\lambda\,\mathrm d|\Lambda|$ be the polar
decomposition of $\Lambda$ with $\|\lambda(\bar h)\|_{\mathcal V}=1$ $|\Lambda|\text{-a.e.}$ Then
\[
T
=
\frac12
\int_{B_{\overline{\mathcal U}}}
\lambda(\bar h)\otimes h\,
\mathrm d|\Lambda|(\bar h),
\]
where $(v\otimes h)u:=\langle u,h\rangle_{\mathcal U}v$. 
Since $\mathcal{U}$ is separable, $\bar h=(b,h)\mapsto h$ is strongly measurable.
Together with the strong measurability of $\lambda$ and the continuity of
\[
(v,h)\mapsto v\otimes h:\mathcal{V}\times \mathcal{U}\to S_1(\mathcal{U},\mathcal{V}),
\]
this shows that $\bar h\mapsto \lambda(\bar h)\otimes h$ is strongly measurable. Moreover,
\[
\|\lambda(\bar h)\otimes h\|_{S_1(\mathcal{U},\mathcal{V})}
=
\|\lambda(\bar h)\|_{\mathcal{V}}\|h\|_{\mathcal{U}}
\leq 1
\qquad |\Lambda|\text{-a.e.}
\]
Hence the integral is well defined in $S_1(\mathcal{U},\mathcal{V})$.
Consequently,
\[
\|T\|_{\mathcal S_1(\mathcal U,\mathcal V)}
\le
\frac12
\int_{B_{\overline{\mathcal U}}}
\|h\|_{\mathcal U}\,\mathrm d|\Lambda|(\bar h)
\le
\frac12\|\Lambda\|_{\mathrm{TV}}.
\]
Taking the infimum over all representing measures $\Lambda$ gives
\[
2\|T\|_{\mathcal S_1(\mathcal U,\mathcal V)}
\le
\|T\|_{\mathbb V(\mathcal U;\mathcal V)},
\]
which proves the result.

\section{Proof of Proposition \ref{prop: Delta_1}}
\label{Proof of prop: Delta_1}

Fix \(F\in\mathbb V_n(Q)\), and choose a representing measure
\(\Lambda\in\mathcal M(B_{\overline{\mathcal U}};V_n)\) such that
\begin{equation}  \label{temp15}
    F(u)=
\int_{B_{\overline{\mathcal U}}}
\sigma(\langle \bar u,\bar h\rangle_{\overline{\mathcal U}})
\,d\Lambda(\bar h),
\qquad
|\Lambda|(B_{\overline{\mathcal U}})\leq Q.
\end{equation}
By the polar decomposition of vector measures,
\begin{equation} \label{temp16}
    d\Lambda(\bar h)=\lambda(\bar h)\,d|\Lambda|(\bar h),
\qquad
\|\lambda(\bar h)\|_{\mathcal V}=1
\quad |\Lambda|\text{-a.e.}
\end{equation}
Hence
\begin{equation} \label{temp10}
    \begin{aligned}
\|F(u)\|_{\mathcal V}^2
=
\int\!\!\int
&\sigma(\langle \bar u,\bar h\rangle_{\overline{\mathcal U}})
\sigma(\langle \bar u,\bar h'\rangle_{\overline{\mathcal U}})
\\
&\times
\langle \lambda(\bar h),\lambda(\bar h')\rangle_{\mathcal V}
\,d|\Lambda|(\bar h)\,d|\Lambda|(\bar h').
\end{aligned}
\end{equation}

We first estimate \(\mathbb E\Delta_1(Q)\) by symmetrization. Let
\(u_1',\ldots,u_K'\) be an independent copy of \(u_1,\ldots,u_K\), and
denote by \(\mathbb E'\) the expectation with respect to
\((u_k')_{k=1}^K\). For \(F\in\mathbb V_n(Q)\), define
\[
f_F(u)=\|F(u)\|_{\mathcal V}^2
\mathbf 1_{\{\omega(u)\leq R\}}.
\]
By the definition of \(\Delta_1(Q)\) and Jensen's inequality,
\begin{equation*}
    \begin{aligned}
        \mathbb E\Delta_1(Q)\;=\;&\mathbb E\sup_{F\in\mathbb V_n(Q)}
\frac1K\left|
\sum_{k=1}^K
f_F(u_k)
-
\mathbb E'\left[
f_F(u_k')
\right]
\right|
\\\;\leq\;&\mathbb E\mathbb E'\sup_{F\in\mathbb V_n(Q)}
\frac1K\left|
\sum_{k=1}^K
f_F(u_k)
-
f_F(u_k')
\right|.
    \end{aligned}
\end{equation*}
Let \(r_1,\ldots,r_K\) be independent Rademacher variables, independent of
both samples. Since each pair \((u_k,u_k')\) is exchangeable,
\begin{equation} \label{temp11}
    \begin{aligned}
        \mathbb E\Delta_1(Q)\;\leq&\;\mathbb E\mathbb E'\mathbb E_r\sup_{F\in\mathbb V_n(Q)}
\frac1K\left|
\sum_{k=1}^K
r_k\l(f_F(u_k)
-
f_F(u_k')
\r)
\right|
\\\;\leq&\;\frac{2}{K}\mathbb E\mathbb E_r\sup_{F\in\mathbb V_n(Q)}\left|
\sum_{k=1}^Kr_kf_F(u_k)
\right|.
    \end{aligned}
\end{equation}
By \eqref{temp10},
\begin{equation}\label{temp12}
\begin{aligned}
\left|
\sum_{k=1}^K r_k f_F(u_k)
\right|
&=
\Bigg|
\iint
\left(
\sum_{k=1}^K r_k
\sigma(\langle \bar u_k,\bar h\rangle_{\overline{\mathcal U}})
\sigma(\langle \bar u_k,\bar h'\rangle_{\overline{\mathcal U}})
\mathbf 1_{\{\omega(u_k)\le R\}}
\right)
\\
&\qquad
\times
\langle \lambda(\bar h),\lambda(\bar h')\rangle_{\mathcal V}
\,d|\Lambda|(\bar h)\,d|\Lambda|(\bar h')
\Bigg|
\\
&\le
Q^2\sup_{\bar h,\bar h'\in B_{\overline{\mathcal U}}}
\left|
\sum_{k=1}^K r_k
\sigma(\langle \bar u_k,\bar h\rangle_{\overline{\mathcal U}})
\sigma(\langle \bar u_k,\bar h'\rangle_{\overline{\mathcal U}})
\mathbf 1_{\{\omega(u_k)\le R\}}
\right|.
\end{aligned}
\end{equation}
Combining \eqref{temp11} and \eqref{temp12} gives
\begin{equation} \label{temp14}
    \mathbb E\Delta_1(Q)\;\leq\;\frac{2Q^2}{K}
\mathbb E\mathbb E_r
\sup_{\bar h,\bar h'\in B_{\overline{\mathcal U}}}
\left|
\sum_{k=1}^K
r_k
\sigma(\langle \bar u_k,\bar h\rangle_{\overline{\mathcal U}})
\sigma(\langle \bar u_k,\bar h'\rangle_{\overline{\mathcal U}})
\mathbf 1_{\{\omega(u_k)\leq R\}}
\right|.
\end{equation}

Set $\widetilde u_k:=\bar u_k\mathbf 1_{\{\omega(u_k)\leq R\}}.$ Since \(\sigma(0)=0\),
\[
\sigma(\langle \bar u_k,\bar h\rangle_{\overline{\mathcal U}})\mathbf 1_{\{\omega(u_k)\leq R\}}
=
\sigma(\langle \widetilde u_k,\bar h\rangle_{\overline{\mathcal U}}),
\qquad
\|\widetilde u_k\|_{\overline{\mathcal U}}\leq R.
\]
On \([-R,R]^2\), the map $(a,b)\mapsto \sigma(a)\sigma(b)$ is \(\sqrt{2}L_\sigma^2R\)-Lipschitz and vanishes at \((0,0)\). 

To apply Lemma~\ref{lem:vector-contraction}, choose a countable weakly dense
subset \(D\subset B_{\overline{\mathcal U}}\) containing \(0\), and define
\[
\mathcal T
:=
\left\{
\left(
\big(
\langle\widetilde u_k,\bar h\rangle_{\overline{\mathcal U}},
\langle\widetilde u_k,\bar h'\rangle_{\overline{\mathcal U}}
\big)
\right)_{k=1}^K
:
\bar h,\bar h'\in D
\right\}.
\]
Since \(B_{\overline{\mathcal U}}\) is weakly compact and metrizable, such a
set \(D\) exists, and continuity in \((\bar h,\bar h')\) shows that restricting
the supremum to \(D\times D\) does not change its value.

Applying Lemma~\ref{lem:vector-contraction} conditionally on
\((u_k)_{k=1}^K\), we obtain
\[
\begin{aligned}
&\mathbb E_r
\sup_{\bar h,\bar h'\in B_{\overline{\mathcal U}}}
\left|
\sum_{k=1}^K
r_k
\sigma(\langle\widetilde u_k,\bar h\rangle_{\overline{\mathcal U}})
\sigma(\langle\widetilde u_k,\bar h'\rangle_{\overline{\mathcal U}})
\right|
\\
&\qquad\leq
2\sqrt{\pi}L_\sigma^2R\,
\mathbb E_\gamma
\sup_{\bar h,\bar h'\in B_{\overline{\mathcal U}}}
\sum_{k=1}^K
\left(
\gamma_{k,1}
\langle\widetilde u_k,\bar h\rangle_{\overline{\mathcal U}}
+
\gamma_{k,2}
\langle\widetilde u_k,\bar h'\rangle_{\overline{\mathcal U}}
\right)
\\
&\qquad=
2\sqrt{\pi}L_\sigma^2R\,
\mathbb E_\gamma
\left[
\left\|
\sum_{k=1}^K\gamma_{k,1}\widetilde u_k
\right\|_{\overline{\mathcal U}}
+
\left\|
\sum_{k=1}^K\gamma_{k,2}\widetilde u_k
\right\|_{\overline{\mathcal U}}
\right].
\end{aligned}
\]
Therefore, by \eqref{temp14},
\[
\mathbb E\Delta_1(Q)
\leq
\frac{8\sqrt{\pi}L_\sigma^2Q^2R}{K}
\mathbb E\mathbb E_\gamma
\left\|
\sum_{k=1}^K\gamma_k\widetilde u_k
\right\|_{\overline{\mathcal U}},
\]
where \(\gamma_1,\ldots,\gamma_K\) are independent standard Gaussian
variables. By the Cauchy--Schwarz inequality,
\begin{equation} \label{temp13}
    \begin{aligned}
        \mathbb E\Delta_1(Q)
\leq& \frac{8\sqrt{\pi} L_\sigma^2Q^2R}{K}\sqrt{\mathbb E \mathbb  E_\gamma
\left\|
\sum_{k=1}^K \gamma_k\widetilde u_k
\right\|_{\overline{\mathcal U}}^2}
\\=&\frac{8\sqrt{\pi} L_\sigma^2Q^2R}{\sqrt{K}}\sqrt{\mathbb{E}\|\widetilde u_1\|_{\overline{\mathcal U}}^2} \leq \frac{8\sqrt{\pi} L_\sigma^2Q^2R^2}{\sqrt{K}}.
    \end{aligned}
\end{equation}

It remains to pass from the expectation bound to a high-probability bound.
For \((u_1,\ldots,u_K)\in\mathcal U^K\), define 
\[
\Phi(u_1,\ldots,u_K)
:=
\sup_{F\in\mathbb V_n(Q)}
\left|
\frac1K\sum_{k=1}^K f_F(u_k)
-
\mathbb E f_F(u)
\right|.
\]
Thus $\Phi(u_1,\ldots,u_K)=\Delta_1(Q).$
For fixed \(F\in\mathbb V_n(Q)\), write
\[
A_F
:=
\frac1K\sum_{k=1}^K f_F(u_k)
-
\mathbb E f_F(u).
\]
If the \(j\)-th observation \(u_j\) is replaced by \(u_j'\), denote the
corresponding quantity by \(A_F'\). Since
\[
0\leq f_F(u)\leq L_\sigma^2Q^2R^2,
\]
we have
\[
|A_F-A_F'|
=
\frac{|f_F(u_j)-f_F(u_j')|}{K}
\leq
\frac{L_\sigma^2Q^2R^2}{K}.
\]
Consequently,
\[
\begin{aligned}
\left|
\sup_{F\in\mathbb V_n(Q)}|A_F|
-
\sup_{F\in\mathbb V_n(Q)}|A_F'|
\right|
&\leq
\sup_{F\in\mathbb V_n(Q)}
\bigl||A_F|-|A_F'|\bigr|
\\
&\leq
\sup_{F\in\mathbb V_n(Q)}
|A_F-A_F'|
\\
&\leq
\frac{L_\sigma^2Q^2R^2}{K}.
\end{aligned}
\]
Hence Lemma~\ref{lem:bounded-differences} applies with
\[
\mathcal X_k=\mathcal U,
\qquad
X_k=u_k,
\qquad
c_k=\frac{L_\sigma^2Q^2R^2}{K},
\qquad k=1,\ldots,K.
\]
Taking
\[
t
=
L_\sigma^2Q^2R^2
\sqrt{\frac{\log(6/\delta)}{2K}},
\]
Lemma~\ref{lem:bounded-differences} and \eqref{temp13} imply that, with
probability at least \(1-\delta/3\),
\[
\begin{aligned}
\Delta_1(Q)
&\leq
\mathbb E\Delta_1(Q)+t
\\
&\leq
\frac{8\sqrt{\pi}L_\sigma^2Q^2R^2}{\sqrt K}
+
L_\sigma^2Q^2R^2
\sqrt{\frac{\log(6/\delta)}{2K}}
\\
&\leq
15L_\sigma^2Q^2R^2
\sqrt{\frac{\log(12/\delta)}{K}}.
\end{aligned}
\]
This completes the proof.


\section{Proof of Proposition \ref{prop: Delta_2}}
\label{Proof of prop: Delta_2}

We follow the symmetrization argument used in the proof of
Proposition~\ref{prop: Delta_1}, omitting analogous details. For
\(F\in\mathbb V_n(Q)\), define
\[
f_F(u,\xi)
:=
\langle F(u),\xi\rangle_{\mathcal V}
\mathbf 1_{\{\omega(u)\leq R,\ \|\xi\|_{\mathcal V}\leq S\}}.
\]
Let \((u_k',\xi_k')_{k=1}^K\) be an independent copy of
\((u_k,\xi_k)_{k=1}^K\), and let \(\mathbb E'\) denote the corresponding
expectation. Introducing independent Rademacher variables
\(r_1,\ldots,r_K\), we obtain
\begin{equation}\label{temp17}
\begin{aligned}
\mathbb E\Delta_2(Q)
&=
\mathbb E
\sup_{F\in\mathbb V_n(Q)}
\left|
\frac1K\sum_{k=1}^K f_F(u_k,\xi_k)
-
\mathbb E f_F(u,\xi)
\right|
\\
&\leq
\mathbb E\mathbb E'
\sup_{F\in\mathbb V_n(Q)}
\left|
\frac1K\sum_{k=1}^K
\bigl(
f_F(u_k,\xi_k)-f_F(u_k',\xi_k')
\bigr)
\right|
\\
&=
\frac1K
\mathbb E\mathbb E'\mathbb E_r
\sup_{F\in\mathbb V_n(Q)}
\left|
\sum_{k=1}^K
r_k
\bigl(
f_F(u_k,\xi_k)-f_F(u_k',\xi_k')
\bigr)
\right|
\\
&\leq
\frac{2}{K}
\mathbb E\mathbb E_r
\sup_{F\in\mathbb V_n(Q)}
\left|
\sum_{k=1}^K r_k f_F(u_k,\xi_k)
\right|.
\end{aligned}
\end{equation}

For \(F\in\mathbb V_n(Q)\), choose \(\Lambda\) and \(\lambda\) as in
\eqref{temp15}--\eqref{temp16}. Then
\begin{equation}\label{temp18}
\begin{aligned}
\left|
\sum_{k=1}^K r_k f_F(u_k,\xi_k)
\right|
&=
\left|
\int_{B_{\overline{\mathcal U}}}
\left\langle
\sum_{k=1}^K
r_k
\sigma(\langle\bar u_k,\bar h\rangle_{\overline{\mathcal U}})
\xi_k
\mathbf 1_{\{\omega(u_k)\leq R,\ \|\xi_k\|_{\mathcal V}\leq S\}},
\lambda(\bar h)
\right\rangle_{\mathcal V}
\,d|\Lambda|(\bar h)
\right|
\\
&\leq
Q
\sup_{\bar h\in B_{\overline{\mathcal U}}}
\left\|
\sum_{k=1}^K
r_k
\sigma(\langle\bar u_k,\bar h\rangle_{\overline{\mathcal U}})
\xi_k
\mathbf 1_{\{\omega(u_k)\leq R,\ \|\xi_k\|_{\mathcal V}\leq S\}}
\right\|_{\mathcal V}.
\end{aligned}
\end{equation}

Define $\widetilde u_k
:=
\bar u_k\mathbf 1_{\{\omega(u_k)\leq R\}}$ and $\widetilde\xi_k
:=
\xi_k\mathbf 1_{\{\|\xi_k\|_{\mathcal V}\leq S\}}$. Since \(\sigma(0)=0\), combining \eqref{temp17} and \eqref{temp18} gives
\begin{equation}\label{temp20}
\begin{aligned}
\mathbb E\Delta_2(Q)
&\leq
\frac{2Q}{K}
\mathbb E\mathbb E_r
\sup_{\bar h\in B_{\overline{\mathcal U}}}
\left\|
\sum_{k=1}^K
r_k
\sigma(\langle\widetilde u_k,\bar h\rangle_{\overline{\mathcal U}})
\widetilde\xi_k
\right\|_{\mathcal V}
\\
&=
\frac{2Q}{K}
\mathbb E\mathbb E_r
\sup_{\bar h\in B_{\overline{\mathcal U}}}
\sup_{v\in B_{\mathcal V}}
\left|
\sum_{k=1}^K
r_k
\sigma(\langle\widetilde u_k,\bar h\rangle_{\overline{\mathcal U}})
\langle\widetilde\xi_k,v\rangle_{\mathcal V}
\right|.
\end{aligned}
\end{equation}

Consider the normalized index set
\[
\mathcal T'
:=
\left\{
\left(
\left(
\frac{\langle\widetilde u_k,\bar h\rangle_{\overline{\mathcal U}}}{R},
\frac{\langle\widetilde\xi_k,v\rangle_{\mathcal V}}{S}
\right)
\right)_{k=1}^K
:
\bar h\in B_{\overline{\mathcal U}},
\ v\in B_{\mathcal V}
\right\}.
\]
Since \(B_{\overline{\mathcal U}}\) and \(B_{\mathcal V}\) are weakly compact
and metrizable, we choose countable weakly dense subsets $D_{\overline{\mathcal U}}\subset B_{\overline{\mathcal U}}$ and $D_{\mv}\subset B_{\mathcal V}$, each containing zero. By continuity,
the corresponding countable subset \(\mathcal T\subset\mathcal T'\) yields
the same suprema as \(\mathcal T'\).

Define
\[
\phi(a,b):=S\sigma(Ra)b,
\qquad (a,b)\in[-1,1]^2.
\]
Then \(\phi(0,0)=0\), and
\[
|\phi(a,b)-\phi(a',b')|
\leq
\sqrt{2}L_\sigma RS
\sqrt{|a-a'|^2+|b-b'|^2}.
\]
Applying Lemma~\ref{lem:vector-contraction} conditionally on
\((u_k,\xi_k)_{k=1}^K\), we obtain
\begin{equation}\label{temp19}
\begin{aligned}
&\mathbb E_r
\sup_{\bar h\in B_{\overline{\mathcal U}}}
\sup_{v\in B_{\mathcal V}}
\left|
\sum_{k=1}^K
r_k
\sigma(\langle\widetilde u_k,\bar h\rangle_{\overline{\mathcal U}})
\langle\widetilde\xi_k,v\rangle_{\mathcal V}
\right|
\\
&\qquad\leq
2\sqrt{\pi}L_\sigma RS\,
\mathbb E_\gamma
\sup_{\bar h\in B_{\overline{\mathcal U}}}
\sup_{v\in B_{\mathcal V}}
\sum_{k=1}^K
\left(
\gamma_{k,1}
\frac{\langle\widetilde u_k,\bar h\rangle_{\overline{\mathcal U}}}{R}
+
\gamma_{k,2}
\frac{\langle\widetilde\xi_k,v\rangle_{\mathcal V}}{S}
\right),
\end{aligned}
\end{equation}
where \(\{\gamma_{k,1},\gamma_{k,2}\}_{k=1}^K\) are independent standard
Gaussian random variables.

Combining \eqref{temp19} with \eqref{temp20} and using the
Cauchy--Schwarz inequality yields
\[
\begin{aligned}
\mathbb E\Delta_2(Q)
&\leq
\frac{4\sqrt{\pi}L_\sigma QRS}{K}
\mathbb E
\left[
\frac{1}{R}
\mathbb E_\gamma
\left\|
\sum_{k=1}^K\gamma_{k,1}\widetilde u_k
\right\|_{\overline{\mathcal U}}
+
\frac{1}{S}
\mathbb E_\gamma
\left\|
\sum_{k=1}^K\gamma_{k,2}\widetilde\xi_k
\right\|_{\mathcal V}
\right]
\\
&\leq
\frac{4\sqrt{\pi}L_\sigma QS}{K}
\left(
\mathbb E\,\mathbb E_\gamma
\left\|
\sum_{k=1}^K\gamma_{k,1}\widetilde u_k
\right\|_{\overline{\mathcal U}}^2
\right)^{1/2}
\\
&\quad+
\frac{4\sqrt{\pi}L_\sigma QR}{K}
\left(
\mathbb E\,\mathbb E_\gamma
\left\|
\sum_{k=1}^K\gamma_{k,2}\widetilde\xi_k
\right\|_{\mathcal V}^2
\right)^{1/2}
\\
&\leq
\frac{8\sqrt{\pi}L_\sigma QRS}{\sqrt K}.
\end{aligned}
\]

It remains to pass from the expectation bound to a high-probability bound.
Let
\[
Z_k:=(u_k,\xi_k)\in\mathcal U\times\mathcal V,
\qquad k=1,\ldots,K.
\]
The random variables \(Z_1,\ldots,Z_K\) are independent. Moreover, by the estimate
\[
\|F(u)\|_{\mathcal V}\leq L_\sigma Q\omega(u),
\]
we have $|f_F(u,\xi)|\leq L_\sigma QRS.$ 
For \(F\in\mathbb V_n(Q)\), write
\[
A_F
:=
\frac1K\sum_{k=1}^K f_F(u_k,\xi_k)
-
\mathbb E f_F(u,\xi).
\]
If \(Z_j\) is replaced by an independent copy \(Z_j'\), and \(A_F'\) denotes
the corresponding quantity, then
\[
|A_F-A_F'|
\leq
\frac{2L_\sigma QRS}{K}.
\]
Consequently,
\[
\begin{aligned}
\left|
\sup_{F\in\mathbb V_n(Q)}|A_F|
-
\sup_{F\in\mathbb V_n(Q)}|A_F'|
\right|
&\leq
\sup_{F\in\mathbb V_n(Q)}
|A_F-A_F'|
\\
&\leq
\frac{2L_\sigma QRS}{K}.
\end{aligned}
\]
Hence Lemma~\ref{lem:bounded-differences} applies with
\[
\mathcal X_k=\mathcal U\times\mathcal V,
\qquad
X_k=Z_k,
\qquad
c_k=\frac{2L_\sigma QRS}{K},
\qquad k=1,\ldots,K.
\]
Taking
\[
t
=
L_\sigma QRS
\sqrt{\frac{2\log(6/\delta)}{K}},
\]
Lemma~\ref{lem:bounded-differences} implies that, with probability at least
\(1-\delta/3\),
\[
\begin{aligned}
\Delta_2(Q)
&\leq
\mathbb E\Delta_2(Q)+t
\\
&\leq
\frac{8\sqrt{\pi}L_\sigma QRS}{\sqrt K}
+
L_\sigma QRS
\sqrt{\frac{2\log(6/\delta)}{K}}
\\
&\leq
16L_\sigma QRS
\sqrt{\frac{\log(12/\delta)}{K}}.
\end{aligned}
\]
This completes the proof.

\bibliographystyle{plain}
\bibliography{ref}

\end{document}